%% file: main3.tex
\documentclass[10pt,journal,compsoc]{IEEEtran}

\input{math_commands}

\usepackage{subfigure}
\usepackage{booktabs}
\usepackage{enumitem}
\usepackage{makecell}
\usepackage{algorithm}
\usepackage{amsfonts, amssymb}
\usepackage{algorithmic}
\usepackage{mathtools}
\usepackage{amsthm}
\usepackage{setspace}
\usepackage{bm}
\usepackage{graphicx}
\usepackage{rotating}
\usepackage{multirow}
\usepackage{amssymb} 
\usepackage{hyperref}
\usepackage{cleveref}
\usepackage[utf8]{inputenc}
\usepackage{ragged2e}
\usepackage[english]{babel}
\usepackage{chngcntr}
\usepackage{color}

\counterwithout{figure}{section}
\counterwithout{table}{section}

\newtheorem{theorem}{Theorem}

\newtheorem{prop}[theorem]{Proposition}
\newtheorem{lemma}[theorem]{Lemma}
\newtheorem{definition}{Definition}
\newtheorem{assumption}{Assumption}

\renewcommand{\thefigure}{\arabic{figure}}
\renewcommand{\thetable}{\arabic{table}}

\hypersetup{hidelinks,
	colorlinks=true,
	allcolors=black,
	pdfstartview=Fit,
	breaklinks=true}

\usepackage{color}
\newcommand{\syf}[1]{%
  \textcolor{black}{#1}%
}

\ifCLASSOPTIONcompsoc
  \usepackage[nocompress]{cite}
\else
  \usepackage{cite}
\fi
\ifCLASSINFOpdf
\else
\fi
\begin{document}
\renewcommand\arraystretch{0.5}
%
\title{Time Series Domain Adaptation via Latent Invariant Causal Mechanism}
%
%
%
%

\author{Ruichu Cai,
        Junxian Huang,
        Zhenhui Yang,
        Zijian Li,
        Emadeldeen Eldele,
        Min Wu,
        Fuchun Sun, \textit{Fellow, IEEE}
\IEEEcompsocitemizethanks{
\IEEEcompsocthanksitem Ruichu Cai is with the School of Computer Science, Guangdong University of Technology, Guangzhou, China, 510006 and Peng Cheng Laboratory, Shenzhen, China, 518066.
Email: cairuichu@gmail.com\protect
\IEEEcompsocthanksitem Junxian Huang is with the School of Computer Science, Guangdong University of Technology, Guangzhou China, 510006.\protect Email: huangjunxian459@gmail.com\protect
\IEEEcompsocthanksitem
Zhenhui Yang is with the School of Computing, Guangdong University of Technology, Guangzhou China, 510006.\protect
E-mail: yangzhenhui8@gmail.com
\IEEEcompsocthanksitem Zijian Li is with the Machine Learning Department, Mohamed bin Zayed University of Artificial Intelligence, Abu Dhabi.\protect Email: leizigin@gmail.com\protect

\IEEEcompsocthanksitem Emadeldeen Eldele is with the Institute for Infocomm Research, A*STAR, Singapore, Centre for Frontier AI Research, A*STAR, Singapore \protect Email: emad0002@ntu.edu.sg\protect
\IEEEcompsocthanksitem Min Wu is with the Institute for Infocomm Research, A*STAR, Singapore. E-mail: wumin@i2r.a-star.edu.sg\protect
\IEEEcompsocthanksitem Fuchun Sun is with the Department of Computer Science and Technology, Tsinghua University, Beijing, China. E-mail: fcsun@tsinghua.edu.cn\protect
}
\thanks{This research was supported in part by the National Key R\&D Program of China (2021ZD0111501), National Science Fund for Excellent Young Scholars (62122022), Natural Science Foundation of China (61876043, 61976052), the major key project of PCL (PCL2021A12). (Min Wu is the Corresponding author.)}}

%
%

\markboth{Journal of \LaTeX\ Class Files,~Vol.~14, No.~8, August~2015}%
{Shell \MakeLowercase{\textit{et al.}}: Bare Demo of IEEEtran.cls for Computer Society Journals}
%



\IEEEtitleabstractindextext{%
\begin{abstract}\justifying
Time series domain adaptation aims to transfer the complex temporal dependence from the labeled source domain to the unlabeled target domain. Recent advances leverage the stable causal mechanism over observed variables to model the domain-invariant temporal dependence. However, modeling precise causal structures in high-dimensional data, such as videos, remains challenging. Additionally, direct causal edges may not exist among observed variables (e.g., pixels). These limitations hinder the \textcolor{black}{applicability} of existing approaches to real-world scenarios. To address these challenges, we find that the high-dimension time series data are generated from the low-dimension latent variables, which motivates us to model the causal mechanisms of the temporal latent process. Based on this intuition, we propose a latent causal mechanism identification framework that guarantees the uniqueness of the reconstructed latent causal structures. Specifically, we first identify latent variables by utilizing sufficient changes in historical information. Moreover, by enforcing the sparsity of the relationships of latent variables, we can achieve identifiable latent causal structures. Built on the theoretical results, we develop the Latent Causality Alignment (\textbf{LCA}) model that leverages variational inference, which incorporates an intra-domain latent sparsity constraint for latent structure reconstruction and an inter-domain latent sparsity constraint for domain-invariant structure reconstruction. Experiment results on eight benchmarks show a general improvement in the domain-adaptive time series classification and forecasting tasks, highlighting the effectiveness of our method in real-world scenarios. Codes are available at \url{https://github.com/DMIRLAB-Group/LCA}.





\end{abstract}

}

\maketitle

\IEEEdisplaynontitleabstractindextext

%

\IEEEpeerreviewmaketitle

\IEEEraisesectionheading{\section{Introduction}\label{sec:introduction}}

To overcome the challenges of distribution shift between the training and the test time series datasets \cite{pan2009survey}, time series domain adaptation seeks to transfer the temporal dependence from labeled source data to unlabeled target data. Mathematically, in the context of time series domain adaptation, we let $X=(\rvx_1,\cdots,\rvx_{\tau},\cdots,\rvx_t)$ be multivariate time series with $t$ timestamps \textcolor{black}{and a channel size of $n$}, where $\rvx_{\tau} \in \mathbb{R}^n$ and $Y$ is the corresponding label\textcolor{black}{, such that $Y$ can be a time series or scalar, depending on the forecasting or classification tasks.} By considering $\rvu$ as the domain label, we further assume that $P(X, Y|\rvu^S)$ and $P(X, Y|\rvu^T)$ are the source and target distributions, respectively. In the source domain, we can access $m^S$ annotated $X$-$Y$ pairs, represented as $(X^S, Y^S)=(X^S_{i}, Y^S_{i})_{i=1}^{m^S}$. While in the target domain, only $m_T$ unannotated time series data can be observed, denoted by $(X^T)=(X^T_{i})_{i=1}^{m_T}$. The primary goal of unsupervised time series domain adaptation is to model the target joint distribution $P(X,Y|\rvu^T)$ by leveraging the labeled source and the unlabeled target data.


Several methods have been proposed to address this problem. Previous works have extended the conventional assumptions of domain adaptation for static data \cite{zhang2015multi,cai2019learning} to time series data, including covariate shift \cite{sugiyama2007covariate}, label shift \cite{lipton2018detecting}, and conditional shift \cite{zhang2013domain}. For instance, leveraging the covariate shift assumption, i.e., $P(Y|X)$ is stable, \textcolor{black}{some researchers \cite{da2020remaining,purushotham2022variational}} employ recursive neural network-based architecture as feature extractor and adopt adversarial training strategy to extract domain-invariant information. \textcolor{black}{Others, e.g., Wilson et al.} \cite{wilson2023calda}, utilize adversarial and contrastive learning to extract the domain-invariant representation for time series data. Moreover, other researchers assume that the distribution shift of $Y$ varies across different domains, known as label shift. Specifically, \textcolor{black}{He et al.} \cite{he2023domain} tackle the feature shift and label shift by aligning both temporal and frequency features. Conversely, \textcolor{black}{Hoyez et al. \cite{hoyeztime}} argue that the content shift and style shift in time series data belong to the conditional shift, where $P(X|Y)$ varies with different domains in time series domain adaptation.

\begin{figure*}[h]
    \centering
    \includegraphics[width=0.95\textwidth]{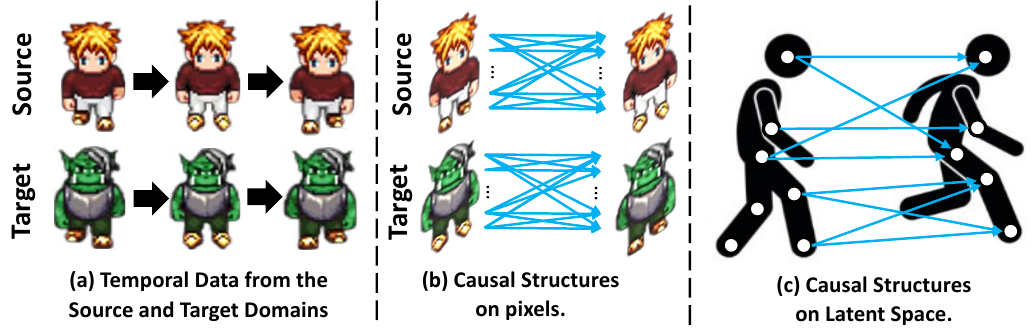}
    \caption{A toy domain adaptation example for video data, where the blue arrows denote the estimated causal relationships. (a) Two three-frame videos, one featuring a walking human and the other a monster, represent the source and target domains, respectively. (b) Since the skeleton variables are latent confounders and there are no causal relationships among pixels, the existing methods may learn dense and fault causal relationships among adjacent two frames, failing to extract the domain-invariant causal mechanism. (c) By identifying the latent variables like the joints in the skeleton, it is convenient to model the latent invariant causal mechanisms behind observed variables. Hence, we can address the time series domain adaptation problem in complex real-world scenarios.}
    \label{fig:motivation}
\end{figure*}
Instead of introducing assumptions from a statistical perspective, another promising direction is to harness the invariant temporal dependencies for time series domain adaptation. Specifically, Cai et al. \cite{Cai_Chen_Li_Chen_Zhang_Ye_Li_Yang_Zhang_2021} address time series domain adaptation by aligning sparse associate structures across different domains. \textcolor{black}{Li et al. }\cite{LI2024106659} further extend this approach by considering the domain-specific strengths of the domain-invariant sparse associative structures. Recently, Wang et al. \cite{wang2023sensor} proposed reducing distribution shift at both local and global sensor levels by aligning sensor relationships. Moreover, \textcolor{black}{Wang et al.} \cite{wang2024sea++} exploit multi-graph-based higher-order sensor alignment to achieve more robust adaptation. To further utilize domain-invariant temporal causal dependencies, \textcolor{black}{Li et al. \cite{li2023transferable}} proposed leveraging stable causal structures over observed variables, introducing the concept of causal conditional shift.

Although existing methods \cite{li2023transferable} demonstrate effective transfer performance on time series by extracting domain-invariant causal mechanisms, they implicitly assume that the dimension of observed variables is small and that direct causal edges exist among these variables. Specifically, most of the existing methods for Granger causal discovery \cite{9376668,khanna2019economy,marcinkevivcs2021interpretable,li2023transferable} are designed for low-dimension time series data (e.g., $n \leq 20$). However, many real-world time series datasets, such as electricity load diagrams or weather, are high-dimensional, limiting the applicability of these methods. 

Furthermore, direct causal relationships may not usually exist in these high-dimension data. For instance,  Figure \ref{fig:motivation} (a) provides a toy example of three-frame video data, where the walking human and monster denote the source and target domains, respectively. As shown in Figure \ref{fig:motivation} (b), existing methods like \cite{li2023transferable}, that estimate Granger causal structure over observed variables, may generate incorrect and dense causal structures since there are no causal relationships in the pixel level. As a result, these methods fail to extract domain-invariant causal mechanisms, limiting their effectiveness in addressing the time series domain adaptation problem.


To address the aforementioned problems, an intuitive solution is based on the observation that videos depicting walking motion from different domains are generated from a domain-invariant skeleton, as shown in Figure \ref{fig:motivation} (c). This insight motivates us to exploit the causal structures over latent variables for time series domain adaptation. Building on this intuition, we prove that the latent causal process can be uniquely reconstructed, i.e., achieving identifiability, with the aid of nonlinear independent component analysis. 

Guided by these theoretical results, we develop a latent causality alignment model (\textbf{LCA} in short). Specifically, the proposed \textbf{LCA} model employs variational inference to model the time series data from the source and target domains, utilizing a flow-based prior estimator. Moreover, we incorporate a gradient-based sparsity constraint to discover the Granger causal structures on latent variables. Furthermore, we harness the latent causality regularization to restrict the discrepancy of the causal structures from different domains. Our method is validated on eight time-series domain adaptation datasets, including video and high-dimension electricity load diagram datasets. The impressive performance, outperforming state-of-the-art methods, demonstrates the effectiveness of our approach.


\textcolor{black}{Our contributions can be summarized as follows:}
\begin{itemize}
    \item \textcolor{black}{Breaking out the limitation of causality alignment on observed variables, we develop a latent causal structure alignment method for time-series domain adaptation from the view of causal representation learning. we are the first to employ causal representation learning to time series domain adaptation.}
    \item \textcolor{black}{Different from previous works that leverage causality structure for time series domain adaptation, we provide formal identification guarantees for both the causal representation and the latent causal structures.}
    \item \textcolor{black}{We propose a general framework for time series domain adaptation, validated through extensive experiments on time series classification and forecasting tasks across eight datasets. The results consistently demonstrate significant performance improvements over state-of-the-art methods.} 
\end{itemize}

\section{Related Works}\label{related_works}




\subsection{Unsupervised Domain Adaptation}
Unsupervised domain adaptation \cite{cai2019learning,kong2022partial,shui2021aggregating,stojanov2021domain,wen2019bayesian,zhang2013domain} aims to leverage the knowledge from a labeled source domain to an unlabeled target domain, by training a model to domain-invariant representations \cite{bousmalis2016domain}. Researchers have adopted different directions to tackle this problem. For example, Long et al. \cite{long2017deep} trained the model to minimize a similarity measure, i.e., maximum mean discrepancy (MMD), to guide learning domain-invariant representations. Tzeng et al. \cite{tzeng2014deep} used an adaptation layer and a domain confusion loss. Another direction is to assume the stability of conditional distributions across domains and extract the label-wise domain-invariant representation \cite{chen2019joint,chen2019progressive,kang2020contrastive}. For instance, Xie et al. \cite{xie2018learning} constrained the label-wise domain discrepancy, and Shu et al. \cite{shu2018dirt} considered that the decision boundaries should not cross high-density data regions, so they propose the virtual adversarial domain adaptation model. Another type of assumption is the target shift \cite{lipton2018detecting,roberts2022unsupervised,wen2020domain,zhang2013domain}, which assumes $p(Y|\rvu)$ varies across different domains. 

Besides, several methods address the domain adaptation problem from a causality perspective. Specifically, Zhang et al. \cite{zhang2013domain} proposed the target shift, conditional shift, and generalized target shift assumptions, based on the premise that $p(Y)$ and $P(X|Y)$ vary independently. Cai et al. \cite{cai2019learning} leveraged the data generation process to extract the disentangled semantic representations. Building on causality analysis, Petar et al. \cite{stojanov2021domain} highlighted the significance of incorporating domain-specific knowledge for learning domain-invariant representation. Recently, Kong et al. \cite{kong2022partial} addressed the multi-source domain adaptation by identifying the latent variables, and Li et al. \cite{li2024subspace} further relaxed the identifiability assumptions.

\subsection{Domain Adaptation on Temporal Data}
In recent years, domain adaptation for time series data has garnered significant attention. Da Costa et al. \cite{da2020remaining} is one of the earliest time series domain adaptation works, where authors adopted techniques originally designed for non-time series data to this domain, incorporating recurrent neural networks as feature extractors to capture domain-invariant representations. Purushotham et al. \cite{purushotham2022variational} further refined this approach by employing variational recurrent neural networks \cite{chung2015recurrent} to enhance extracting domain-invariant features. However, such methods face challenges in effectively capturing domain-invariant information due to the intricate dependencies between time points.
Subsequently, Cai et al. \cite{cai2021time} proposed the Sparse Associative Structure Alignment (SASA) method, based on the assumption that sparse associative structures among variables remain stable across domains. This method has been successfully applied to adaptive time series classification and regression tasks. Additionally, Jin et al. \cite{jin2022domain} introduced the Domain Adaptation Forecaster (DAF), which leverages statistical relationships from relevant source domains to improve performance in target domains. Li et al. \cite{li2023transferable} hypothesized that causal structures are consistent across domains, leading to the development of the Granger Causality Alignment (GCA) approach. This method uncovers underlying causal structures while modeling shifts in conditional distributions across domains.

Our work also relates to domain adaptation for video data, which could be considered a form of high-dimensional time series data. Video data offers a robust benchmark for evaluating the performance of our method. Unsupervised domain adaptation for video data has recently attracted substantial interest. For instance, Chen et al. \cite{chen2019temporal} proposed a Temporal Attentive Adversarial Adaptation Network (TA3N), which integrates a temporal relation module to enhance temporal alignment. Choi et al. \cite{choi2020shuffle} proposed the SAVA method, which leverages self-supervised clip order prediction and attention-based alignment across clips. In addition, Pan et al. proposed the Temporal Co-attention Network (TCoN) \cite{pan2020adversarial}, which employs a cross-domain co-attention mechanism to identify key frames shared across domains, thereby improving alignment. 

Luo et al. \cite{luo2020adversarial} focused on domain-agnostic classification using a bipartite graph network topology to model cross-domain correlations. Rather than relying on adversarial learning, Sahoo et al. \cite{sahoo2021contrast} developed CoMix, an end-to-end temporal contrastive learning framework that employs background mixing and target pseudo-labels. More recently, Chen et al. \cite{chen2022multi} introduced multiple domain discriminators for multi-level temporal attentive features to achieve superior alignment, while Turrisi et al.~\cite{da2022dual} utilized a dual-headed deep architecture that combines cross-entropy and contrastive losses to learn a more robust target classifier. Additionally, Wei et al.~\cite{wei2023unsupervised} employed contrastive and adversarial learning to disentangle dynamic and static information in videos, leveraging shared dynamic information across domains for more accurate prediction.

\subsection{Granger Causal Discovery}
Several works have been raised to infer causal structures for time series data based on Granger causality \cite{diks2006new,granger1969investigating,marcinkevivcs2021interpretable,lowe2020amortized,lin2024root,gong2023causal}. Previously, several researchers used the vector autoregressive (\textbf{VAR}) model \cite{lozano2009grouped,hamilton2020time} with the sparsity constraint like Lasso or Group Lasso \cite{yuan2006model,tibshirani1996regression} to learn Granger causality. Recently, several works have inferred Granger causality with the aid of neural networks. For instance, Tank et al. \cite{tank2021neural} developed a neural network-based autoregressive model with sparsity penalties applied to network weights. Inspired by the interpretability of self-explaining neural networks, Marcinkevivcs et al. \cite{marcinkevivcs2021interpretable} introduced a generalized vector autoregression model to learn Granger causality. Li et al. \cite{li2023transferable} considered the Granger causal structure as latent variables. Cheng et al. \cite{cheng2023cuts,cheng2024cuts+} proposed a neural Granger causal discovery algorithm to discover Granger causality from irregular time series data. Lin et al. \cite{lin2024root} used a neural architecture with contrastive learning to learn Granger causality. However, these methods usually consider the Granger causal structures over low-dimension observed time series data, which can hardly address the time series data with high dimension or latent causal relationships. To address this limitation, we identify the latent variables and infer the latent Granger causal structures behind high-dimensional time series data.

\subsection{Identifiability of Generative Model}
To achieve causal representation \cite{rajendran2024learning,mansouri2023object,wendong2024causal} for time series data, many researchers leverage Independent Component Analysis (ICA) to recover latent variables with identifiability guarantees \cite{yao2023multi,scholkopf2021toward,Liu2023CausalTriplet,gresele2020incomplete}. Conventional methods typically assume a linear mixing function from the latent variables to the observed variables \cite{comon1994independent,hyvarinen2013independent,lee1998independent,zhang2007kernel}. To relax the linear assumption, researchers achieve the identifiability of latent variables in nonlinear ICA by using different types of assumptions, such as auxiliary variables or sparse generation processes \cite{zheng2022identifiability,hyvarinen1999nonlinear,hyvarinen2023identifiability,khemakhem2020ice,li2023identifying}. Aapo et al. \cite{hyvarinen2017nonlinear} first achieved identifiability for methods employing auxiliary variables by assuming the latent sources follow an exponential family distribution and introducing auxiliary variables, such as domain indices, time indices, and class labels. To further relax the exponential family assumption, Zhang et al. \cite{kong2022partial,xie2022multi, kong2023identification,yan2023counterfactual, xie2022multi} proposed component-wise identifiability results for nonlinear ICA, requiring $2n+1$ auxiliary variables for $n$ latent variables. To seek identifiability in an unsupervised manner, researchers employed the assumption of structural sparsity to achieve identifiability \cite{ng2024identifiability,lachapelle2022disentanglement,zheng2022identifiability, xu2024sparsity}. Recently, Zhang et al. \cite{zhang2024causal} achieved identifiability under distribution shift by leveraging the sparse structures of latent variables. Li et al. \cite{li2024identification} further employed sparse causal influence to achieve identifiability for time series data with instantaneous dependency.


\section{PRELIMINARIES}

In this section, we first describe the data generation process from multiple distributions under latent Granger Causality. Sequentially, we further provide the definition of generalized causal conditional shift as well as identifiability of latent causal process.

\subsection{Data Generation Process from Multiple Distributions under Latent Granger Causality}
To illustrate how we address time series domain adaptation with latent causality alignment, we first consider the data generation process under latent Granger causality. Specifically, we first let $X=(\rvx_1,\cdots,\rvx_{\tau}, \cdots, \rvx_t)$ be multivariate time series with $t$ time steps. Each $\rvx_{\tau} \in \mathbb{R}^m$ is generated from latent variables $\rvz_{\tau} \in \mathbb{R}^n, m \gg n$ via an invertible nonlinear
mixing function $\rvg$ as follows:
\begin{equation}
\label{equ:g1}
    \rvx_{t}=\rvg(\rvz_{t}).
\end{equation}
Moreover, the $i$-th dimension of the latent variables $\rvz_{t,i}$ is generated through the latent causal process, which is assumed to be influenced by the time-delayed parent variables $\text{Pa}(\rvz_{t,i})$ and the different domains, simultaneously. Formally, this relationship can be expressed using
a structural equation model as follows:
\begin{equation}
\label{equ:g2}
\begin{split}
    z_{t,i}= f_i(\text{Pa}(z_{t,i}), \rvu ,\epsilon_{t,i}) \quad \text{with} \quad \epsilon_{t,i}\sim p_{\epsilon_{i}},
\end{split}
\end{equation}
where $\epsilon_{\tau,i}$ denotes the temporally and spatially
independent noise extracted from a distribution $p_{\epsilon_{i}}$, and $\rvu$ denotes the domain index that could be either the source ($S$) or target ($T$) domains.

\subsection{Generalized Causal Conditional Shift Assumption}

Inspired by the domain-invariant causal mechanism between the source and target domains, we generalize the causal conditional shift assumption \cite{li2023transferable} from observed variables to latent variables, which is discussed next.
\begin{assumption}
(\textbf{Generalized Causal Conditional Shift}) Given latent causal structure $\mA$ and latent variables $\rvz_1,\cdots,\rvz_{t}, \rvz_{t+1}$, we assume that the conditional distribution $P(\rvz_{t+1}|\rvz_1,\cdots,\rvz_{t})$ varies with different domains, while the latent causal structures are stable across different domains. This can be formalized as follows:

\begin{equation}
\begin{split}
    P(\rvz_{t+1}|\rvz_1,\cdots,\rvz_{t},S)&\neq P(\rvz_{t+1}|\rvz_1,\cdots,\rvz_{t},T)\\
    A_S&=A_T,
\end{split}
\end{equation}
where $A_S$ and $A_T$ are the causal structures over latent variables from the source and target domains, respectively.
\end{assumption}

Based on the aforementioned assumption, we consider both the forecasting and the classification time series domain adaptation tasks. Starting with the time series forecasting, the label $Y=(\rvx_{t+1},\cdots,\rvx_{t+\rho})$ denotes the values of the future $\rho$ time steps. On the other hand, $Y$ denotes the discrete class labels for the classification tasks, which are generated as follows:
\begin{equation}
    Y=C(\rvz_1,\cdots,\rvz_t),
\end{equation}
where $C$ is the labeling function. Our objective in this paper is to use the labeled source data and unlabeled target data to identify the target joint distribution.

\subsection{Identifiability of Target Joint Distribution}
In this subsection, we discuss the mechanism of identifying the target joint distribution. By introducing the latent variables and combining the data generation process, we can factorize the target joint distribution as shown in Equation (\ref{eq: prob_int}).

\begin{equation}
\label{eq: prob_int}
\begin{split}
    &P(X,Y|T)=\int_{Z}P(X|Z)P(Y|Z)P(Z|T)dZ,
\end{split}
\end{equation}
where $Z=(\rvz_1,\cdots,\rvz_{t})$ denotes the latent causal process. According to the aforementioned equation, we can identify the target joint distribution by modeling the conditional distribution of observed data given latent variables and identifying the latent causal process under theoretical guarantees. Specifically, if estimated latent processes can be uniquely identified up to permutation and component-wise invertible transformations, then the latent causal relationships are also identifiable. This can be regarded to the conditional independence relations that fully characterize causal relationships within a causal system, assuming the absence of causal confounders within latent causal processes.

\begin{definition}[Identifiable Latent Causal Process]
\label{def: identifibility}
Let $\mathbf{X} = \{\mathbf{x}_1, \dots, \mathbf{x}_l\}$ represent a sequence of observed variables generated according to the true latent causal processes characterized by $(f_i, p(\epsilon_i), \mathbf{g})$, as described in Equations (\ref{equ:g1}) and (\ref{equ:g2}). A learned generative model $(\hat{f}_i, \hat{p}(\epsilon_i), \hat{\mathbf{g}})$ is said to be observationally equivalent to $(f_i, p(\epsilon_i), \mathbf{g})$ if the distribution $p_{\hat{f}_i, \hat{p}(\epsilon_i), \hat{\mathbf{g}}}(\{\mathbf{x}\}_{t=1}^l)$ produced by the model matches the true data distribution $p_{f_i, p(\epsilon_i), \mathbf{g}}(\{\mathbf{x}\}_{t=1}^l)$ for all possible values of $\mathbf{x}_t$. We say that the latent causal processes are identifiable if such observational equivalence implies that the generative model can be transformed into the true generative process by means of a permutation $\pi$ and a component-wise invertible transformation $\mathcal{T}$:
\begin{equation}
\label{eq:iden}
     p_{\hat{f}_i, \hat{p}_{\epsilon_i},\hat{\mathbf{g}}}\big(\{\mathbf{x}\}_{t=1}^T\big) = p_{f_i, p_{\epsilon_i}, \mathbf{g}} \big(\{\mathbf{x}\}_{t=1}^T\big) 
     \ \ \Rightarrow\ \   \hat{\mathbf{g}} = \mathbf{g}  \circ  \pi  \circ \mathcal{T}.
\end{equation}
\end{definition} 

Upon the identification of the mixing process, the latent variables will be identified up to a permutation and a component-wise invertible transformation:
\begin{align}
    \mathbf{\hat{z}}_t 
    &= \mathbf{\hat{g}}^{-1}(\mathbf{x}_t) 
    = \big(\mathcal{T}^{-1} \circ \pi^{-1} \circ \mathbf{g}^{-1}\big)(\mathbf{x}_t) \\
 \nonumber
    &= \big(\mathcal{T}^{-1} \circ \pi^{-1} \big)\big(\mathbf{g}^{-1}(\mathbf{x}_t)\big) 
    = \big(\mathcal{T}^{-1} \circ \pi^{-1} \big)(\mathbf{z}_t).
\end{align}


\section{\textcolor{black}{Identifying Latent Causal Process}}
Based on the definition of the identification of latent causal processes, we demonstrate how to identify the latent causal process within the context of a sparse latent process. Specifically, we first utilize the connection between conditional independence and cross derivatives \cite{lin1997factorizing}, along with the sufficient variability of temporal data, to uncover the relationships between the estimated and true latent variables. Furthermore, we establish the identifiability of the latent variables by imposing constraints on sparse causal influences, as shown in Lemma \ref{Th1}.


\syf{
We demonstrate that given certain assumptions, the latent variables are also identifiable. To integrate contextual information for identifying latent variables $\rvz_t$, we consider latent variables across $L$ consecutive timestamps, including $\rvz_t$. For simplicity, we analyze the case where the sequence length is 2 (i.e., $L=2$) and the time lag is 1. 
}

\begin{lemma}
\label{Th1}
\textbf{(Identifiability of Temporally Latent Process)} \cite{yao2022temporally} Suppose there exists invertible function $\hat{\mathbf{g}}$ that maps $\mathbf{x}_t$ to $\hat{\mathbf{z}}_t$, i.e., $\hat{\mathbf{z}}_t=\hat{\mathbf{g}}(\mathbf{x}_t)$, such that the components of  $\hat{\mathbf{z}}_t$ are mutually independent conditional on $\hat{\mathbf{z}}_{t-1}$.Let
\begin{equation}
\label{eq: independent_condition}
 \begin{split}
    \mathbf{v}_{t,k} =
    \Big(\frac{\partial^{2}\log p(z_{t,k}|\mathbf{z}_{t-1}) }{\partial z_{t,k}\partial z_{t-1,1}},\frac{\partial^{2}\log p(z_{t,k}|\mathbf{z}_{t-1})}{\partial z_{t,k}\partial z_{t-1,2}},...,\\
        \frac{\partial^{2}\log p(z_{t,k}|\mathbf{z}_{t-1})}{\partial z_{t,k}\partial z_{t-1,n}}\Big)^{\mathsf{T}},\\
    \mathring{\mathbf{v}}_{t,k}=
\Big(\frac{\partial^{3}\log p(z_{t,k}|\mathbf{z}_{t-1})}{\partial z_{t,k}^{2}\partial z_{t-1,1}},\frac{\partial^{3}\log p(z_{t,k}|\mathbf{z}_{t-1})}{\partial z_{t,k}^{2}\partial z_{t-1,2}},...,\\ 
    \frac{\partial^{3}\log p(z_{t,k}|\mathbf{z}_{t-1})}{\partial z_{t,k}^{2}\partial z_{t-1,n}}\Big)^{\mathsf{T}}.
    \end{split}
\end{equation}
If for each value of $\mathbf{z}_t,\mathbf{v}_{t,1},\mathring{\mathbf{v}}_{t,1},\mathbf{v}_{t,2},\mathring{\mathbf{v}}_{t,2},...,,\mathbf{v}_{t,n},\mathring{\mathbf{v}}_{t,n}$, as 2n vector function in $z_{t-1,1},z_{t-1,2},...,z_{t-1,n}$, are linearly independent, then $\mathbf{z}_t$ must be an invertible, component-wise transformation of a permuted version of $\hat{\mathbf{z}}_t$.
\end{lemma}
\textbf{Proof Sketch.} The proof can be found in Appendix A.1. First, we establish a bijective transformation between the ground truth $\rvz_t$ and the estimated $\hat{\rvz}_t$ to connect them. Sequentially, we leverage the historical information to construct a full-rank linear system, where the only solution of $\frac{\partial z_{t,i}}{\partial \hat{z}_{t,j}} \cdot \frac{\partial z_{t,i}}{\partial \hat{z}_{t,k}}$ is zero. Since the Jacobian of $h$ is invertible, for each $z_{t,i}$, there exists a $h_i$ such that $z_{t,i}=h(\hat{z}_{t,j})$, implying that $z_{t,i}$ is component-wise identifiable.

By identifying the latent variables, they can be considered as observed variables up to a permutation and invertible transformation. As a result, Granger Causality among $\rvz_t$ and $\rvz_{t-1}$ can be identified by using sparsity constraint on the relationships among latent variables, which is shown in Proposition \ref{prop1} (also refer to the proof in Appendix A.3 for more details).

\begin{prop}
\label{prop1}
\textbf{(Identification of Granger Causality among latent variables)} Suppose the estimated transition function $f_i$ accurately models the relationship between $\rvz_{t}$ and $\rvz_{t-1}$. Let the ground truth Granger Causal structure be represented as $\mathcal{G}=(V, E_V)$ with the maximum lag of 1, where $V$ and $E_V$ denote the nodes and edges, respectively. If the data are generated according to Equation (\ref{equ:g1}) and (\ref{equ:g2}), and sparsity is enforced among the latent variables, then $z_{t-1,j}\rightarrow z_{t,i} \notin E_V$ if and only if $\frac{\partial z_{t,i}}{\partial z_{t-1,j}}=0$. 
\end{prop}

\section{Latent Causality Alignment Model}
Based on the theoretical results, we propose the latent causality alignment (\textbf{LCA}), shown in Figure \ref{fig:model}, for time series domain adaptation. This figure shows a variational-inference-based neural architecture to model time series data, a prior estimation network, and a downstream neural forecaster for different downstream tasks.
\subsection{Temporally Variational Inference Architecture}
\begin{figure*}
	\centering
	\includegraphics[width=1.9\columnwidth]{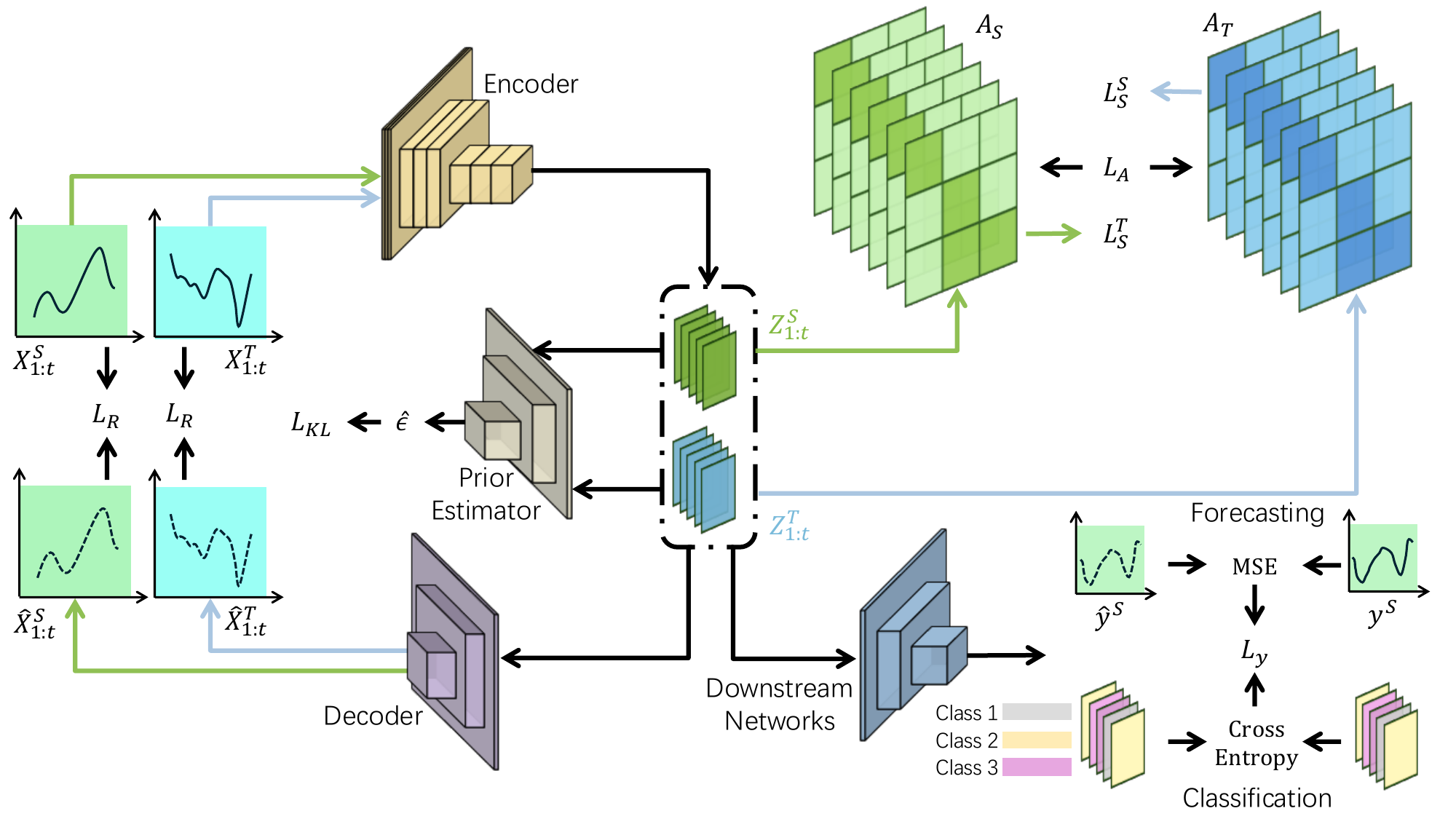}
    \caption{Model architecture: Blue arrows indicate the flow of source data, while purple arrows represent the flow of target data. The loss function is highlighted in bold as \( L \). Subfigures (a) and (b) depict the architectures for prediction and classification tasks, respectively.}
    \label{fig:model}
\end{figure*}
We first derive the evidence lower bound (ELBO) to model the time series data, as follows: 
\begin{equation}
\label{eq:elbop}
\begin{split}
&\ln P(X, Y) = \ln {P(\rvx_{1:t}, Y)}=\ln\frac{P(\rvx_{1:t}, Y, \rvz_{1:t})}{P(\rvz_{1:t}|\rvx_{1:t}, Y)}\\\geq&\underbrace{\mathbb{E}_{Q(\rvz_{1:t}|\rvx_{1:t})}\ln P(\rvx_{1:t}|\rvz_{1:t})}_{L_{R}} +\underbrace{\mathbb{E}_{Q(\rvz_{1:t})}\ln P(Y|\rvz_{1:t})}_{L_{Y}} \\
&\underbrace{
-D_{KL}(Q(\rvz_{1:t}|\rvx_{1:t})||P(\rvz_{1:t}))
}_{L_{KL}} ,
\end{split}
\end{equation}
where \(Q(\rvz_{1:t}|\rvx_{1:t})\) and \(P(\rvx_{1:t}|\rvz_{1:t})\) are used to approximate the prior distribution of latent variables and reconstruct the observations, respectively. Technologically, we consider $Q(\rvz_{1:t}|\rvx_{1:t})$ and $P(\rvx_{1:t}|\rvz_{1:t})$ as the encoder and decoder networks, respectively, which can be formalized as follows:
\begin{equation}
    \hat{\rvz}_{1:t} = \psi(\rvx_{1:t}) \quad \hat{\rvx}_{1:t} = \phi(\rvx_{1:t}),
\end{equation}
where $\psi$ and $\phi$ denote the encoder and decoder respectively.
\subsection{Prior Estimation Networks} \label{sec:prior}
To estimate the prior distribution $P(\rvz_{1:t})$, we propose the prior estimation networks. Specifically, we first let ${r_i}$ be a set of learned inverse transition functions that receive the estimated latent variables \textcolor{black}{with the superscript symbol $\hat{}$} as input, and use it to estimate the noise term $\hat{\epsilon}_i$, i.e., $\hat{\epsilon}_{t,i}=r_i(\hat{\rvz}_{t,i},\hat{\rvz}_{t-1})$, where each $r_i$ is implemented by Multi-layer Perceptron networks (MLPs). Sequentially, we devise a transformation $\kappa := \{\hat{\rvz}_{t-1},\hat{\rvz}_t\}\rightarrow\{\hat{\rvz}_{t-1},\hat{\bm{\epsilon}}_{t}\}$, whose Jacobian can be formalized as ${\mathbf{J}_{\kappa}=
    \begin{pmatrix}
        \mathbb{I}&0\\
        \mathbf{J}_d & \mathbf{J}_e
    \end{pmatrix}}$, where $\mathbf{J}_d(i,j) = \frac{\partial r_i}{\partial \hat{z}_{t-1,j}}
$ and $\mathbf{J}_e\!\!=\!\!\text{diag}\!\left(\frac{\partial r_i}{\partial \hat{z}_{t,i}}\!\right)$. Hence we have Equation (\ref{eq:pri_1}) via the change of variables formula. 
\begin{equation}
\label{eq:pri_1}
\small
    \log p(\hat{\rvz}_t, \hat{\rvz}_{t-1})=\log p(\hat{\rvz}_{t-1},\epsilon_t) + \log |\frac{\partial r_i}{\partial z_{t,i}}|.
\end{equation}
According to the generation process, the noise term \(\epsilon_{t,i}\) is independent of \(\rvz_{t-1}\). Therefore, we can impose independence on the estimated noise term \(\hat{\epsilon}_{t,i}\). Consequently, Equation~(\ref{eq:pri_1}) can be further expressed as:
\begin{equation}
\small
    \log p(\hat{\rvz}_t \mid \{\hat{\rvz}_{t-\tau}\}) = \sum_{i=1}^{n} \log p(\hat{\epsilon}_{t,i}) + \sum_{i=1}^{n} \log \left| \frac{\partial r_i}{\partial \hat{z}_{t,i}} \right|.
\end{equation}
Assuming a \(\tau\)-order Markov process, the prior \( p(\hat{\rvz}_{1:T}) \) can be expressed as \( p(\hat{z}_{1:\tau}) \prod_{t=\tau+1}^{T} p(\hat{\rvz}_{t} \mid \{ \hat{\rvz}_{t-\tau} \}) \). Hence, the log-likelihood \( \log p(\hat{\rvz}_{1:T}) \) can be estimated as:
\begin{equation}
\small
    \log p(\hat{\rvz}_{1:T}) = \log p(\hat{z}_{1:\tau}) + \sum_{t=\tau+1}^T \left( \sum_{i=1}^n \log p(\hat{\epsilon}_{t,i}) + \sum_{i=1}^n \log \left| \frac{\partial r_i}{\partial \hat{z}_{t,i}} \right| \right).
\end{equation}
Here, we assume that the noise term \( p(\hat{\epsilon}_{\tau,i}) \) and the initial latent variables \( p(\hat{\rvz}_{1:\tau}) \) follow Gaussian distributions. For \(\tau = 1\), this simplifies to:
\begin{equation}
\small
    \log p(\hat{\rvz}_{1:T}) = \log p(\hat{z}_{1}) + \sum_{t=2}^T \left( \sum_{i=1}^n \log p(\hat{\epsilon}_{t,i}) + \sum_{i=1}^n \log \left| \frac{\partial r_i}{\partial \hat{z}_{t,i}} \right| \right).
\end{equation}
Similarly, the distribution \( p(\hat{\rvz}_{t+1:T} \mid \hat{\rvz}_{1:t}) \) can be estimated using analogous methods. For further details on the derivation of the prior, please refer to Appendix B.

\subsection{Sparsity Constraint for Latent Granger Causality} 
\begin{figure}
	\centering
	\includegraphics[width=0.9\columnwidth]{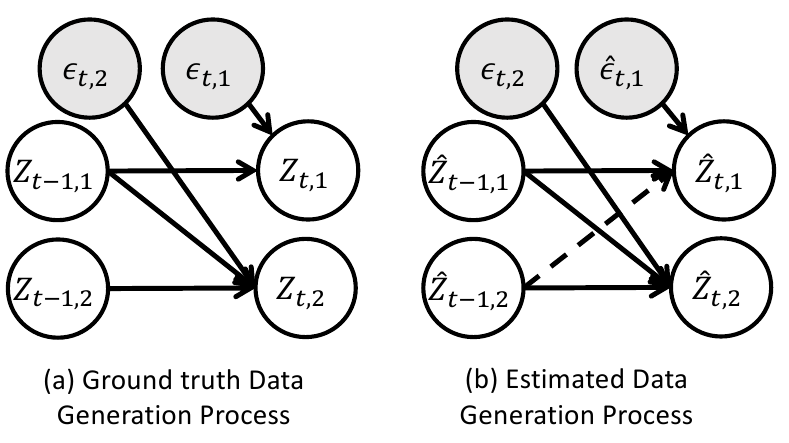}
    \caption{To describe the relationship between partial derivatives and the existence of edges clearly, we assume that the data follow a simple linear generation process.}
    \label{fig: noise_meaning}
\end{figure}
Based on the theoretical results, we can identify the latent variables by using the variational-influenced-based architecture and prior estimation networks. However, without any further constraints, it is hard for us to infer the causal structures over latent variables. To address this problem, we propose the sparsity constraint on partial derivatives regarding independent noise and latent variables for the latent Granger Causality. To provide a clearer understanding of its implications, we use a straightforward example with ground truth and estimated generation processes as shown in Figure \ref{fig: noise_meaning}, where the boiled and dashed arrows denote the ground truth and spurious relationships. According to Figure \ref{fig: noise_meaning}, $z_{t,1}$ and $\hat{z}_{t,1}$ are generated as follows:
\begin{equation}
\begin{split}
    z_{t,1}&=f_1(z_{t-1,1},\epsilon_{t,1}) \\
    \hat{z}_{t,1}&=f_1(\hat{z}_{t-1,1}, \hat{z}_{t-1,2},\hat{\epsilon}_{t,1}),
\end{split}
\end{equation}
where the estimated generation process of $\hat{z}_{t,1}$ includes the spurious dependence from $\hat{z}_{t-1,2}$. To remove these spurious dependencies, we find that the partial derivatives of ground truth generation process $\frac{\partial \epsilon_{t,1}}{\partial z_{t-1,2}} = 0$ and $\frac{\partial \epsilon_{t,1}}{\partial z_{t-1,1}} \neq 0$, meaning that the $\mJ$ can provide an intuitive representation of the causal structures among the latent variables, as it quantifies the influence of \( \hat{\rvz}_{t-1} \) on \( \hat{z}_{t,i} \). Therefore, we can apply sparsity constraint on the partial derivatives regarding the estimated noise term and latent variables, i.e., $\frac{\partial \hat{\epsilon}_{t,1}}{\partial \hat{z}_{t-1,2}}$, to remove the spurious dependencies of latent variables. As a result, we propose to employ the $\mathcal{L}_1$ regularization on the  Jacobian matrix $\mJ$ as shown in Equation (\ref{equ:sparsity}).
\begin{equation}
\label{equ:sparsity}
    L_{S} = ||\mathbf{J}||_1,
\end{equation}
where $||*||_1$ denotes the $\mathcal{L}_1$ Norm of a matrix. 

\subsection{Latent Causality Alignment Constraint}\label{sec:align}

Since the latent structures are stable across different domains, we need to align the latent structures of the source domain and the target domain. However, unlike the structural alignment of the observed variables \cite{li2023transferable}, the alignment of latent variables can be a more challenging task for two main reasons. First, the structure between latent variables is implicit, and the causal structure cannot be directly extracted for alignment as in GCA \cite{li2023transferable}. In addition, although the partial derivatives regarding estimated noise and latent variables can reflect the causal structure over latent variables, directly aligning the gradient can affect the correct gradient descent direction and result in suboptimal performance downstream tasks. To overcome these challenges, we propose the latent causality alignment constraint.

Specifically, we choose a threshold $u$ to determine the causal relations among the latent variables. For instance, if $\mathbf{J}^*_{i,j}>u$, then there is an edge from $z_{t-1, i}$ to $z_{t,j}$, otherwise, no edge is present. Formally, we can obtain the estimated causal relationships of the source or target domains (the superscript * shows the source or target domains) as follows: 
\begin{equation}  
\hat{\mathbf{J}}^*_{i,j} = \left\{
\begin{array}{l}
1, \text{ if }   \mathbf{J}^*_{i,j}>u ;\\
0, \text{ otherwise}.
\end{array}
\right.
\end{equation}
A direct solution to align the latent causal structures is to restrict the discrepancy between the latent structures of source and target domains, following \cite{li2023transferable}. However, since these causal structures are represented using gradients with respect to latent variables, the direct alignment of the gradients can interfere with the model optimization, thereby increasing the difficulty of training. To overcome these issues, we find that it is sufficient to focus on reducing the differences in $\mathbf{J}^*$ between the source and target domains while ignoring the identical parts. This approach achieves causal structure alignment while minimizing the impact on the gradients. Based on this idea, we first obtain a masking matrix through an XOR operation. In this matrix, elements with a value of 0 represent identical structures between the source and target domains, while elements with a value of 1 represent differing structures. We then constrain only the differing parts, which are formalized as follows:
\begin{equation}
\small
\begin{split}
    \mathcal{M} &= (\mathbf{J}^S_d  > u)  \oplus  (\mathbf{J}^T_d  > u) \\
    L_{A} &= ||\mathbf{C}(\mathbf{J}_d^S \odot \mathcal{M})-\mathbf{J}_d^T\odot \mathcal{M}||_1,
\end{split}
\end{equation}
where $\oplus$ and $\odot$ denote the XOR and element-wise product operations, respectively, and $\mathbf{C}(\cdot)$ denotes the gradient-stoping operation \cite{li2023transferable}, which is used to enforce the target latent structures closer to the source latent structures.

By combining Equations (\ref{eq:elbop}) and (\ref{equ:sparsity}), we can achieve the total loss of our method, as follows:
\begin{equation}
    L_{total} = L_Y + \alpha L_R + \beta L_{KL} + \gamma L_S + \delta L_A, 
\label{eqn:overall_loss}
\end{equation}
where $\alpha, \beta, \gamma$, and $\delta$ are tunable hyper-parameters.

\section{EXPERIMENTS}


\begin{table*}[t]
\centering
\caption{MAE and MSE of various methods on the PPG-DaLiA dataset, where R-COAT, iTrans, and TMixer are the abbreviation of RAINCOAT, iTransformer, and TimeMixer, respectively.}
\label{tab:ppg}
\setlength{\tabcolsep}{3mm}
\resizebox{\textwidth}{!}{%
\begin{tabular}{@{}c|c|ccccccccccc@{}}
\toprule
Metric                & Task                        & SASA   & GCA    & DAF    & CLUDA  & R-COAT & AdvSKM & iTrans & TMixer & TSLANet & SegRNN & Ours             \\ \midrule
\multirow{12}{*}{\rotatebox{90}{MSE}} & \multicolumn{1}{c|}{ $\text{C} \rightarrow \text{D}$ } & 0.7421 & 0.8260 & 0.7085 & 0.8869 & 0.9117   & 0.7795 & 0.6036       & 0.5901    & 0.6162  & 0.6333 & \textbf{0.5797} \\
                      & \multicolumn{1}{c|}{$\text{C} \rightarrow \text{S}$} & 0.6279 & 0.6898 & 0.5726 & 0.7851 & 0.8675   & 0.6998 & 0.3433       & 0.2962    & 0.3139  & 0.4259 & \textbf{0.2842} \\
                      & \multicolumn{1}{c|}{$\text{C} \rightarrow \text{W}$} & 0.8767 & 0.9076 & 0.8565 & 0.9709 & 0.9356   & 0.8846 & 0.8117       & 0.8162    & 0.8226  & 0.8452 & \textbf{0.7946} \\ \cmidrule(l){2-13} 
                      & \multicolumn{1}{c|}{D $\rightarrow$ C} & 0.9415 & 0.9858 & 0.9258 & 1.0562 & 1.0134   & 0.9539 & 0.8603       & 0.8476    & 0.8623  & 0.8599 & \textbf{0.8294} \\
                      & \multicolumn{1}{c|}{D $\rightarrow$ S} & 0.3643 & 0.4976 & 0.3751 & 0.5719 & 0.4645   & 0.4179 & 0.3121       & 0.3012    & 0.3128  & 0.3668 & \textbf{0.2697} \\
                      & \multicolumn{1}{c|}{D $\rightarrow$ W} & 0.8666 & 0.9197 & 0.8542 & 0.9844 & 0.9270   & 0.9055 & 0.8365       & 0.8211    & 0.8335  & 0.8366 & \textbf{0.7989} \\ \cmidrule(l){2-13} 
                      & \multicolumn{1}{c|}{S $\rightarrow$ C} & 0.9613 & 0.9568 & 0.9224 & 1.0834 & 1.1211   & 0.9459 & 0.8552       & 0.8360    & 0.8569  & 0.8571 & \textbf{0.8212} \\
                      & \multicolumn{1}{c|}{S $\rightarrow$ D} & 0.5708 & 0.6306 & 0.5756 & 0.8250 & 0.7340   & 0.6727 & 0.5855       & 0.5909    & 0.6178  & 0.5936 & \textbf{0.5358} \\
                      & \multicolumn{1}{c|}{S $\rightarrow$ W} & 0.8660 & 0.9002 & 0.8439 & 1.0098 & 1.0181   & 0.8852 & 0.8310       & 0.8128    & 0.8358  & 0.8283 & \textbf{0.7999} \\ \cmidrule(l){2-13} 
                      & \multicolumn{1}{c|}{W $\rightarrow$ C} & 0.9302 & 0.9786 & 0.8965 & 0.9321 & 0.9396   & 0.8992 & 0.8690       & 0.8340    & 0.8532  & 0.8768 & \textbf{0.8171} \\
                      & \multicolumn{1}{c|}{W $\rightarrow$ D} & 0.7634 & 0.7800 & 0.6983 & 0.8647 & 0.8124   & 0.7736 & 0.6010       & 0.5997    & 0.6287  & 0.6408 & \textbf{0.5577} \\
                      & \multicolumn{1}{c|}{W $\rightarrow$ S} & 0.7244 & 0.6561 & 0.5982 & 0.7484 & 0.8592   & 0.6459 & 0.3365       & 0.3000    & 0.3238  & 0.4535 & \textbf{0.2807} \\ \midrule
\multirow{12}{*}{\rotatebox{90}{MAE}} & \multicolumn{1}{c|}{C $\rightarrow$ D} & 0.5963 & 0.6565 & 0.5817 & 0.7061 & 0.7147   & 0.6371 & 0.5007       & 0.4836    & 0.5001  & 0.5306 & \textbf{0.4678} \\
                      & \multicolumn{1}{c|}{C $\rightarrow$ S} & 0.5139 & 0.5520 & 0.4817 & 0.6354 & 0.6982   & 0.5663 & 0.3300       & 0.2656    & 0.3044  & 0.3956 & \textbf{0.2373} \\
                      & \multicolumn{1}{c|}{C $\rightarrow$ W} & 0.6453 & 0.6750 & 0.6350 & 0.7214 & 0.6983   & 0.6658 & 0.6121       & 0.6138    & 0.6124  & 0.6354 & \textbf{0.5898} \\ \cmidrule(l){2-13} 
                      & \multicolumn{1}{c|}{D $\rightarrow$ C} & 0.6723 & 0.6927 & 0.6597 & 0.7578 & 0.7384   & 0.6820 & 0.6172       & 0.6128    & 0.6237  & 0.6334 & \textbf{0.5942} \\
                      & \multicolumn{1}{c|}{D $\rightarrow$ S} & 0.3620 & 0.4326 & 0.3695 & 0.5130 & 0.4442   & 0.4032 & 0.2782       & 0.2736    & 0.2784  & 0.3465 & \textbf{0.2264} \\
                      & \multicolumn{1}{c|}{D $\rightarrow$ W} & 0.6484 & 0.6733 & 0.6372 & 0.7295 & 0.6983   & 0.6780 & 0.6116       & 0.6066    & 0.6182  & 0.6214 & \textbf{0.5886} \\ \cmidrule(l){2-13} 
                      & \multicolumn{1}{c|}{S $\rightarrow$ C} & 0.6980 & 0.6787 & 0.6858 & 0.7816 & 0.7972   & 0.6918 & 0.6137       & 0.6060    & 0.6139  & 0.6279 & \textbf{0.5920} \\
                      & \multicolumn{1}{c|}{S $\rightarrow$ D} & 0.4814 & 0.5169 & 0.4861 & 0.6663 & 0.5941   & 0.5675 & 0.4750       & 0.4745    & 0.4894  & 0.4873 & \textbf{0.4381} \\
                      & \multicolumn{1}{c|}{S $\rightarrow$ W} & 0.6387 & 0.6583 & 0.6319 & 0.7500 & 0.7344   & 0.6760 & 0.6130       & 0.6051    & 0.6116  & 0.6094 & \textbf{0.5927} \\ \cmidrule(l){2-13} 
                      & \multicolumn{1}{c|}{W $\rightarrow$ C} & 0.6648 & 0.7078 & 0.6448 & 0.6921 & 0.6900   & 0.6675 & 0.6258       & 0.6134    & 0.6233  & 0.6546 & \textbf{0.5911} \\
                      & \multicolumn{1}{c|}{W $\rightarrow$ D} & 0.6060 & 0.6321 & 0.5673 & 0.6922 & 0.6490   & 0.6252 & 0.5032       & 0.4991    & 0.5078  & 0.5391 & \textbf{0.4595} \\
                      & \multicolumn{1}{c|}{W $\rightarrow$ S} & 0.5330 & 0.5557 & 0.4813 & 0.6283 & 0.6249   & 0.5328 & 0.3361       & 0.2858    & 0.3171  & 0.4241 & \textbf{0.2447} \\
\bottomrule
\end{tabular}%
}
\end{table*}


\begin{table*}[]
\caption{MAE and MSE for various methods on the Human Motion dataset.}
\label{tab:human}
\setlength{\tabcolsep}{3mm}
\resizebox{\textwidth}{!}{%
\begin{tabular}{@{}c|c|ccccccccccc@{}}
\toprule
Metric                & Task              & SASA   & GCA    & DAF    & CLUDA  & R-COAT & AdvSKM & iTrans & TMixer & TSLANet & SegRNN & Ours            \\ \midrule
\multirow{12}{*}{\rotatebox{90}{MSE}} & G $\rightarrow$  E & 0.1845 & 0.2701 & 0.2210 & 0.7697 & 0.4857 & 0.4330 & 0.1061 & 0.0611 & 0.0845  & 0.0680 & \textbf{0.0543} \\
                      & G $\rightarrow$ W & 0.1666 & 0.2499 & 0.2033 & 0.6710 & 0.4078 & 0.3856 & 0.1387 & 0.1363 & 0.1354  & 0.1378 & \textbf{0.1066} \\
                      & G $\rightarrow$ S & 0.1405 & 0.1645 & 0.1655 & 0.5964 & 0.3492 & 0.3562 & 0.0630 & 0.0593 & 0.0512  & 0.0599 & \textbf{0.0439} \\ \cmidrule{2-13} 
                      & E $\rightarrow$ G & 0.2220 & 0.2535 & 0.2460 & 0.7796 & 0.5953 & 0.5227 & 0.0806 & 0.0787 & 0.0930  & 0.0821 & \textbf{0.0588} \\
                      & E $\rightarrow$ W & 0.2250 & 0.2353 & 0.2264 & 0.6706 & 0.4090 & 0.4043 & 0.2412 & 0.1414 & 0.1748  & 0.1469 & \textbf{0.1199} \\
                      & E $\rightarrow$ S & 0.1173 & 0.1210 & 0.1247 & 0.5465 & 0.3794 & 0.3204 & 0.0568 & 0.0502 & 0.0456  & 0.0493 & \textbf{0.0399} \\ \cmidrule{2-13} 
                      & W $\rightarrow$ G & 0.2295 & 0.2719 & 0.2754 & 0.8068 & 0.6679 & 0.5823 & 0.0992 & 0.0775 & 0.0840  & 0.0858 & \textbf{0.0582} \\
                      & W $\rightarrow$ E & 0.2183 & 0.2616 & 0.2386 & 0.8086 & 0.6266 & 0.5130 & 0.1382 & 0.0741 & 0.0605  & 0.0965 & \textbf{0.0566} \\
                      & W $\rightarrow$ S & 0.1455 & 0.1506 & 0.1687 & 0.6043 & 0.3765 & 0.3481 & 0.0705 & 0.0525 & 0.0501  & 0.0712 & \textbf{0.0423} \\ \cmidrule{2-13} 
                      & S $\rightarrow$ G & 0.1637 & 0.1958 & 0.1790 & 0.7551 & 0.5090 & 0.3823 & 0.0909 & 0.0767 & 0.0728  & 0.0841 & \textbf{0.0637} \\
                      & S $\rightarrow$ E & 0.1555 & 0.1716 & 0.1505 & 0.7503 & 0.5252 & 0.3510 & 0.0866 & 0.0567 & 0.0677  & 0.0794 & \textbf{0.0564} \\
                      & S $\rightarrow$ W & 0.2012 & 0.2597 & 0.2524 & 0.6825 & 0.4572 & 0.4051 & 0.2182 & 0.1514 & 0.1632  & 0.1419 & \textbf{0.1274} \\ \midrule
\multirow{12}{*}{\rotatebox{90}{MAE}} & G $\rightarrow$ E & 0.2982 & 0.3700 & 0.3307 & 0.6734 & 0.5035 & 0.4933 & 0.2159 & 0.1274 & 0.1656  & 0.1469 & \textbf{0.1146} \\
                      & G $\rightarrow$ W & 0.2816 & 0.3479 & 0.3231 & 0.6339 & 0.4431 & 0.4618 & 0.2043 & 0.1910 & 0.2069  & 0.1958 & \textbf{0.1595} \\
                      & G $\rightarrow$ S & 0.2517 & 0.2751 & 0.2714 & 0.5885 & 0.3939 & 0.4386 & 0.1450 & 0.0975 & 0.1088  & 0.1278 & \textbf{0.0807} \\ \cmidrule{2-13} 
                      & E $\rightarrow$ G & 0.3200 & 0.3499 & 0.3502 & 0.6759 & 0.5531 & 0.5384 & 0.1628 & 0.1601 & 0.1890  & 0.1696 & \textbf{0.1310} \\
                      & E $\rightarrow$ W & 0.3204 & 0.3345 & 0.3301 & 0.6032 & 0.4569 & 0.4661 & 0.3211 & 0.1879 & 0.2171  & 0.2060 & \textbf{0.1658} \\
                      & E $\rightarrow$ S & 0.2383 & 0.2331 & 0.2474 & 0.5641 & 0.4390 & 0.4222 & 0.1280 & 0.0821 & 0.0898  & 0.0950 & \textbf{0.0785} \\ \cmidrule{2-13} 
                      & W $\rightarrow$ G & 0.3164 & 0.3466 & 0.3612 & 0.6755 & 0.5688 & 0.5534 & 0.2058 & 0.1688 & 0.1798  & 0.1674 & \textbf{0.1327} \\
                      & W $\rightarrow$ E & 0.3138 & 0.3566 & 0.3517 & 0.7005 & 0.5800 & 0.5419 & 0.2510 & 0.1606 & 0.1221  & 0.1827 & \textbf{0.1191} \\
                      & W $\rightarrow$ S & 0.2505 & 0.2556 & 0.2738 & 0.5988 & 0.4265 & 0.4373 & 0.1517 & 0.1034 & 0.0989  & 0.1534 & \textbf{0.0819} \\ \cmidrule{2-13} 
                      & S $\rightarrow$ G & 0.2783 & 0.3091 & 0.3004 & 0.6746 & 0.5065 & 0.4577 & 0.1917 & 0.1574 & 0.1521  & 0.1690 & \textbf{0.1369} \\
                      & S $\rightarrow$ E & 0.2799 & 0.2971 & 0.2784 & 0.6591 & 0.4972 & 0.4494 & 0.1812 & 0.1163 & 0.1319  & 0.1526 & \textbf{0.1152} \\
                      & S $\rightarrow$ W & 0.2946 & 0.3423 & 0.3422 & 0.6153 & 0.4572 & 0.4766 & 0.2833 & 0.1949 & 0.2007  & 0.1885 & \textbf{0.1689} \\ \bottomrule
\end{tabular}
}
\end{table*}


\begin{table*}[]
\caption{MAE and MSE of various methods on the ETT dataset.}
\label{tab:ett}
\setlength{\tabcolsep}{3mm}
\resizebox{\textwidth}{!}{%
\begin{tabular}{@{}c|c|ccccccccccc@{}}
\toprule
Metric               & Task               & SASA   & GCA    & DAF    & CLUDA  & R-COAT & AdvSKM & iTrans & TMixer & TSLANet & SegRNN & Ours            \\ \midrule
\multirow{2}{*}{MSE} & 1 $\rightarrow$  2 & 0.2843 & 0.2820 & 0.1812 & 0.4097 & 0.3930 & 0.2839 & 0.1439 & 0.1306 & 0.1419  & 0.1775 & \textbf{0.1023} \\
                     & 2 $\rightarrow$ 1  & 0.8068 & 0.8421 & 0.6438 & 0.9711 & 0.8921 & 0.8352 & 0.5848 & 0.6620 & 0.6790  & 0.8740 & \textbf{0.4395} \\ \midrule
\multirow{2}{*}{MAE} & 1 $\rightarrow$  2 & 0.3977 & 0.4188 & 0.3248 & 0.5126 & 0.5083 & 0.4117 & 0.2764 & 0.2585 & 0.2751  & 0.3141 & \textbf{0.2195} \\
                     & 2 $\rightarrow$ 1  & 0.6686 & 0.6670 & 0.5925 & 0.7636 & 0.6796 & 0.6609 & 0.5493 & 0.5696 & 0.5693  & 0.6508 & \textbf{0.4482} \\ \bottomrule
\end{tabular}%

}
\end{table*}

\begin{table*}[]
\caption{MAE and MSE of various methods on the PEMS dataset.}
\label{tab:traffc}
\setlength{\tabcolsep}{3mm}
\resizebox{\textwidth}{!}{%
\begin{tabular}{@{}c|c|ccccccccccc@{}}
\toprule
Metric               & Task              & SASA   & GCA    & DAF    & CLUDA  & R-COAT & AdvSKM & iTrans & TMixer & TSLANet & SegRNN & Ours            \\ \midrule
\multirow{6}{*}{\rotatebox{90}{MSE}} & 1 $\rightarrow$ 2 & 0.3249 & 0.5068 & 0.4412 & 0.4223 & 0.4751 & 0.4176 & 0.4134 & 0.3698 & 0.3558  & 0.3312 & \textbf{0.2629} \\
                     & 1 $\rightarrow$ 3 & 0.3210 & 0.4788 & 0.3806 & 0.4055 & 0.4220 & 0.4045 & 0.3695 & 0.3342 & 0.3055  & 0.2769 & \textbf{0.2244} \\ \cmidrule{2-13} 
                     & 2 $\rightarrow$ 1 & 0.3035 & 0.4866 & 0.4356 & 0.4139 & 0.4024 & 0.3975 & 0.3499 & 0.3023 & 0.2886  & 0.2996 & \textbf{0.2393} \\
                     & 2 $\rightarrow$ 3 & 0.3100 & 0.4383 & 0.3694 & 0.3832 & 0.3858 & 0.3802 & 0.3682 & 0.2980 & 0.2871  & 0.2681 & \textbf{0.2221} \\ \cmidrule{2-13} 
                     & 3 $\rightarrow$ 1 & 0.3423 & 0.5112 & 0.4492 & 0.4627 & 0.4553 & 0.4155 & 0.4064 & 0.3377 & 0.3225  & 0.2952 & \textbf{0.2372} \\
                     & 3 $\rightarrow$ 2 & 0.2894 & 0.4427 & 0.3820 & 0.3620 & 0.4485 & 0.3572 & 0.4083 & 0.3729 & 0.3657  & 0.3515 & \textbf{0.2655} \\ \midrule
\multirow{6}{*}{\rotatebox{90}{MAE}} & 1 $\rightarrow$ 2 & 0.3237 & 0.3988 & 0.3906 & 0.3585 & 0.4041 & 0.3577 & 0.3415 & 0.3398 & 0.3162  & 0.2944 & \textbf{0.2392} \\
                     & 1 $\rightarrow$ 3 & 0.3463 & 0.4209 & 0.3780 & 0.3715 & 0.4044 & 0.3778 & 0.3533 & 0.3393 & 0.3061  & 0.2815 & \textbf{0.2327} \\ \cmidrule{2-13} 
                     & 2 $\rightarrow$ 1 & 0.3332 & 0.4105 & 0.4025 & 0.3727 & 0.3996 & 0.3630 & 0.3308 & 0.3200 & 0.2961  & 0.3058 & \textbf{0.2445} \\
                     & 2 $\rightarrow$ 3 & 0.3385 & 0.3872 & 0.3629 & 0.3532 & 0.3727 & 0.3565 & 0.3472 & 0.3139 & 0.3009  & 0.2848 & \textbf{0.2327} \\ \cmidrule{2-13} 
                     & 3 $\rightarrow$ 1 & 0.3534 & 0.4326 & 0.4063 & 0.3980 & 0.4238 & 0.3813 & 0.3855 & 0.3311 & 0.3310  & 0.2971 & \textbf{0.2425} \\
                     & 3 $\rightarrow$ 2 & 0.3116 & 0.3691 & 0.3496 & 0.3163 & 0.3928 & 0.3164 & 0.3503 & 0.3343 & 0.3245  & 0.3163 & \textbf{0.2433} \\ \bottomrule
\end{tabular}%
}
\end{table*}

\subsection{Domain Adaptation for Time Series Forecasting}

\subsubsection{Datasets}\quad In this section, we provide an overview of the five real-world datasets used to evaluate the \textbf{LCA} model. \textbf{PPG-DaLiA\footnote{\scriptsize 
https://archive.ics.uci.edu/ml/datasets/PPG-DaLiA}} is a publicly available multimodal dataset, used for PPG-based heart rate estimation. It includes physiological and motion data collected from 15 volunteers using wrist and chest-worn devices during various activities. We categorize the data into four domains based on activities: Cycling (C), Sitting (S), Working (W), and Driving (D). \textbf{Human Motion\footnote{\scriptsize 
 http://vision.imar.ro/human3.6m/description.php}} is a dataset for human motion prediction and cross-domain adaptation. We select four motion types as domains: Walking (W), Greeting (G), Eating (E), and Smoking (S). \textbf{Electricity Load Diagrams\footnote{\scriptsize https://archive.ics.uci.edu/dataset/321/electricityloaddiagrams20112014}} is a dataset containing electricity consumption data from 370 substations in Portugal, recorded from January 2011 to December 2014. We account for seasonal domain shifts by dividing the data into four domains based on the months: Domain 1 (January, February, March), Domain 2 (April, May, June), Domain 3 (July, August, September), and Domain 4 (October, November, December). \textbf{PEMS\footnote{\scriptsize 
 https://pems.dot.ca.gov/}} dataset comprises traffic speed data collected by the California Transportation Agencies over a 6-month period from January 1st, 2017 to May 31st, 2017. To account for seasonal domain shifts, we divide the data into three domains based on months: Domain 1 (January), Domain 2 (February), and Domain 3 (March). \textbf{ETT\footnote{\scriptsize 
 https://github.com/zhouhaoyi/ETDataset}} dataset describes the electric power deployment. We use the ETT-small subset, which includes data from two stations, treating each station as a separate domain.
\subsubsection{Baselines}  Here, we introduce the benchmarks for unsupervised domain adaptation in time series forecasting, including SASA\cite{Cai_Chen_Li_Chen_Zhang_Ye_Li_Yang_Zhang_2021}, AdvSKM\cite{liu2021adversarial}, DAF\cite{jin2022domain}, GCA\cite{li2023transferable}, CLUDA\cite{ozyurt2022contrastive}, and Raincoat\cite{he2023domain}. In addition, we include approaches that integrate the Gradient Reversal Layer (GRL) with state-of-the-art time series forecasting techniques such as SegRNN\cite{lin2023segrnn}, TSLANet\cite{tslanet}, TimeMixer\cite{wang2023timemixer}, and iTransformer\cite{liu2023itransformer}, which are based on RNN, CNN, MLP, and Transformer architectures, respectively. For all the above methods, we ensure consistency by employing the same experimental settings among them.

\noindent\subsubsection{Experimental Settings}\quad We performed multivariate time series forecasting across all datasets, using an input length of 30 and an output length of 10. Each dataset is partitioned into train, validation, and test sets. To ensure robustness, we run each method three times with different random seeds and report the average performance. The model yielding the best validation results is selected, and its performance is subsequently assessed on the test set. In all experiments, we employed the ADAM optimizer \cite{kingma2014adam} and used Mean Squared Error (MSE) as the loss function for prediction. We report the Mean Squared Error (MSE) and Mean Absolute Error (MAE) as evaluation metrics. 

\noindent\subsubsection{Quantization Result}\quad We conducted comparative experiments between our method and baseline models across five datasets, and the results are presented in Tables~\ref{tab:ppg},~\ref{tab:human},~\ref{tab:ett},~\ref{tab:traffc},~and~\ref{tab:eld}, respectively. The results show that our method significantly outperforms the baseline models across all metrics, particularly on the Human Motion Dataset, where causal mechanisms are more pronounced (since the relationships between human joints represent a natural causal structure). This further highlights the superiority of our approach.

Additionally, we find that combining state-of-the-art time series forecasting methods with a gradient reversal layer to address UDA in time series forecasting also yielded excellent results. This is clear in the results of iTransformer, TimeMixer, TSLAnet, and SegRNN specifically, where these methods leverage the most advanced models in their model architectures (iTransformer for Transformers, TimeMixer for MLPs, TSLAnet for CNNs, and SegRNN for RNNs). When combined with a gradient reversal layer, these methods effectively capture domain-shared features, enhancing prediction performance. Despite their strong backbones, our approach still outperforms these baselines. Unlike TimeMixer, which captures complex macro and micro temporal information by decomposing trends and seasonal information using a carefully designed MLP, our method uses a simpler MLP learning network, further demonstrating our superiority.

Furthermore, existing UDA methods in the time series domain, such as DAF, CLUDA, Raincoat, and AdvSKM, exhibited 
relatively weaker performance, possibly due to their less optimal backbone models.  However, we notice that SASA and GCA perform better than other time series UDA methods, such as DAF, CLUDA, Raincoat, and AdvSKM. This can be attributed to their consideration of sparse correlation structures among observed variables and the consistency of Granger causality structures, despite having relatively simple backbones. 


\begin{figure}[h]
    \centering
    \includegraphics[width=0.5\textwidth, trim=0.7cm  0cm 0.5cm 0.7cm, clip]{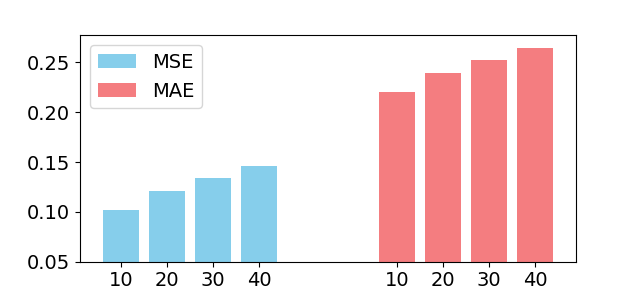}
    \caption{The MSE and MAE values after predicting different lengths in the transition from domain 1 to domain 2 in the ETT dataset. Subfigures (a), (b), (c), and (d) show the forecasting results on lengths of 10, 20, 30, and 40 time steps, respectively.}
    \label{fig:diff_len_v}
\end{figure}

\noindent\textbf{Visualization Result:}\quad To demonstrate the robustness of our method, we performed predictions of varying lengths on the ETT dataset, specifically from domain 1 to domain 2. For each input of the 30 time steps, we predicted future values over 10, 20, 30, and 40 steps. The resulting MSE and MAE values for these different forecast lengths are illustrated in Figure \ref{fig:diff_len_mase}. Notably, our method consistently outperforms most baseline methods that only predict 10 steps, particularly for longer forecast horizons (20, 30, and 40 steps), highlighting the superior performance of our approach. Additionally, we visualize the predictions for the last dimension of the ETT dataset, as shown in Figure \ref{fig:diff_len_v}. The visualization confirms that our method effectively captures the temporal variations in the data.

\subsection{Domain Adaptation for Time Series Classification}
In classification tasks, we experimented on the UCIHAR and HHAR datasets, following AdaTime \cite{ragab2023adatime} framework, which is a benchmarking suite for domain adaptation on time series data. Furthermore, we validate the performance of our method on the high-dimensional video classification datasets.

\begin{table*}[]
\small
\caption{F1-scores of various methods on the UCIHAR time series dataset.}
\resizebox{\textwidth}{!}{%
\begin{tabular}{ccccccccccc}
\toprule
\multicolumn{1}{c|}{Task}       & AdvSKM & CoDATS & CoTimix & DANN   & DDC    & DeepCoral & DIRT   & MMDA   & SASA   & Ours            \\ \midrule
\multicolumn{1}{c|}{18 $\rightarrow$ 14} & 0.9773 & 0.8988 & 0.9232  & 0.7729 & 0.9804 & 0.9902      & 1.0000 & 0.9869 & 0.9869 & \textbf{1.0000} \\ \midrule
\multicolumn{1}{c|}{6 $\rightarrow$ 13}  & 0.9864 & 0.9932 & 0.9386  & 0.7064 & 0.9931 & 1.0000      & 1.0000 & 0.9778 & 0.9935 & \textbf{1.0000} \\ \midrule
\multicolumn{1}{c|}{20 $\rightarrow$ 9}  & 0.3835 & 0.5649 & 0.5527  & 0.3885 & 0.3730 & 0.5871      & 0.5458 & 0.5356 & 0.4860 & \textbf{0.6946} \\ \midrule
\multicolumn{1}{c|}{7 $\rightarrow$ 18}  & 0.7826 & 0.7813 & 0.8429  & 0.4479 & 0.7604 & 0.8246      & 0.8501 & 0.7723 & 0.7414 & \textbf{0.9108} \\ \midrule
\multicolumn{1}{c|}{19 $\rightarrow$ 11} & 0.6765 & 0.9890 & 0.9388  & 0.5348 & 0.6999 & 0.9823      & 0.8651 & 0.9386 & 0.8341 & \textbf{0.9963} \\ \midrule
\multicolumn{1}{c|}{17 $\rightarrow$ 18} & 0.8742 & 0.8730 & 0.9011  & 0.4884 & 0.8767 & 0.9507      & 0.8820 & 0.9191 & 0.7520 & \textbf{0.9601} \\ \midrule
\multicolumn{1}{c|}{9 $\rightarrow$ 19}  & 0.4269 & 0.6956 & 0.8939  & 0.6113 & 0.3869 & 0.7467      & 0.8290 & 0.6182 & 0.6000 & \textbf{0.9692} \\ \midrule
\multicolumn{1}{c|}{2 $\rightarrow$ 12}  & 0.9633 & 0.8191 & 0.8775  & 0.6701 & 0.9857 & 0.9886      & 0.9859 & 0.9963 & 1.0000 & \textbf{1.0000} \\ \midrule
\multicolumn{1}{c|}{12 $\rightarrow$ 3}  & 0.9609 & 0.9744 & 0.9904  & 0.5566 & 0.9678 & 0.9968      & 0.9266 & 0.9936 & 0.9936 & \textbf{1.0000} \\ \midrule
\multicolumn{1}{c|}{17 $\rightarrow$ 14}                     & 0.8804 & 0.7389 & 0.8423  & 0.4021 & 0.8435 & 0.8572      & 0.7738 & 0.9062 & 0.8694 & \textbf{0.9673} \\ \bottomrule
\end{tabular}
}
\label{tab:har}
\end{table*}

\noindent\subsubsection{Datasets}
\noindent\textbf{Time Series Datasets:} \quad We consider the UCIHAR\cite{anguita2013public} dataset, which comprises sensor data from accelerometers, gyroscopes, and body sensors collected from 30 subjects performing six activities: walking, walking upstairs, walking downstairs, standing, sitting, and lying down. Due to the variability between individuals, each subject is treated as a separate domain. We also consider the HHAR\cite{stisen2015smart} dataset, which contains sensor readings from smartphones and smartwatches, collected from 9 subjects, with each subject similarly treated as an individual domain.

\noindent\textbf{Video Datasets:} We adopted two widely recognized and frequently used datasets, i.e., UCF101 and HMDB51. The UCF101 dataset, curated by the University of Central Florida, consists of videos primarily sourced from YouTube, capturing a wide variety of daily activities and sports actions. In contrast, the HMDB51 dataset, collected by the University of Massachusetts Amherst, includes videos from movies, public databases, and YouTube, offering a more diverse range of real-world action scenes. Both datasets serve as key benchmarks for evaluating video classification algorithms. For data processing, we followed the approach used in TranSVAE, extracting relevant and overlapping action categories from both UCF101 and HMDB51. This resulted in a combined dataset of 3,209 videos, where UCF101 provided 1,438 training videos and 571 validation videos, and HMDB51 contributed 840 training videos and 360 validation videos. Based on these datasets, we defined two video-based unsupervised domain adaptation tasks: U → H and H → U.

\noindent\subsubsection{Baselines} \textcolor{black}{Since the time series classification UDA methods use constraints that can be only applied to discrete labels}, we selected several state-of-the-art methods implemented in the AdaTime framework as baselines for comparison. These methods include AdvSKM\cite{liu2021adversarial}, CODATS\cite{wilson2020multi}, CoTMix\cite{eldele2022cotmix}, DANN\cite{ganin2016domain}, DDC\cite{tzeng2014deep}, DeepCoral\cite{sun2017correlation}, DIRT\cite{shu2018dirt}, MMDA\cite{rahman2020minimum}, and SASA\cite{cai2021time}.

For the video classification task, we conducted a comprehensive comparison based on the experimental settings of the latest TranSVAE approach~\cite{wei2023unsupervised}. Specifically, we included five widely-used image-based UDA methods, which perform domain adaptation by disregarding temporal information: DANN~\cite{ganin2016domain}, JAN~\cite{long2017deep}, ADDA~\cite{tzeng2017adversarial}, AdaBN~\cite{li2018adaptive}, and MCD~\cite{sahoo2021contrast}. Additionally, we benchmarked our approach against the latest state-of-the-art video-based UDA methods, including TA3N~\cite{chen2019temporal}, SAVA~\cite{choi2020shuffle}, TCoN~\cite{pan2020adversarial}, ABG~\cite{luo2020adversarial}, CoMix~\cite{sahoo2021contrast}, CO2A~\cite{da2022dual}, and MA2L-TD~\cite{chen2022multi}.

\noindent\subsubsection{Experimental Settings} For each of the time series datasets, we randomly generated 10 transfer directions. Each experiment was run three times and the average F1-score is reported. We generally followed the validation methods provided by AdaTime benchmark.

For video classification, we strictly followed the data processing procedure used in TranSVAE. TranSVAE employs a pretrained I3D model on the Kinetics dataset as the backbone to extract features from each video frame, which are stored and subsequently used as input to the network. To maintain consistency, we directly downloaded these precomputed features from TranSVAE and used them as input for our method. For the comparison methods, we referred to the experimental data from the TranSVAE network, where classification accuracy was recorded. Consistent with the experimental settings for prediction tasks, we employed the Adam optimizer. All experiments were implemented in PyTorch and conducted on a single NVIDIA GTX 3090 GPU with 24GB of memory.

\noindent\subsubsection{Quantization and Visual Results}\quad 
Here, we provide a detailed analysis of our method’s performance on classification tasks.
Tables~\ref{tab:har} and \ref{tab:hhar} present the performance of our proposed method alongside several baseline methods on the UCIHAR and HHAR time series classification datasets. It is worth noting that AdvSKM, CoDATS, DANN, and DIRT are adversarial-based methods, while DDC, Deep-CORAL, and MMDA methods align domains by minimizing a distance metric. On the other hand, CoTMix utilizes contrastive learning, while SASA enforces the consistency of sparse correlation structures to align domains. 

\textcolor{black}{The experimental results on the UCIHAR dataset show that our approach achieves the highest F1-scores across multiple cross-domain scenarios (e.g., 18 → 14, 6 → 13, 20 → 9), reaching a perfect score of 1 in some of them, significantly outperforming other comparative models such as AdvSKM, CoDATS, CoTMix, and DANN. Moreover, our model demonstrates exceptional stability across all tasks, further validating its robustness in temporal classification tasks. On the other temporal classification dataset, HHAR, our method remains the best-performing approach. For the cross-domain task 0 → 5, our method achieved an F1-score of 0.7899, significantly outperforming other comparative models, such as CoDATS (0.5714), MMDA (0.5444), and CoTMix (0.5064). Additionally, our method consistently achieves the best results across other tasks, including 3 → 1, 7 → 4, and 5 → 8.}

\begin{table*}[]
\small
\caption{F1-scores of various methods on the HHAR time series dataset.}
\resizebox{\textwidth}{!}{%

\begin{tabular}{c|cccccccccc}
\toprule
Task   & AdvSKM & CoDATS & CoTMix & DANN   & DDC    & DeepCoral & DIRT   & MMDA   & SASA   & Ours            \\ \midrule
3 $\rightarrow$ 1 & 0.9077 & 0.9582 & 0.9601 & 0.8925 & 0.9120 & 0.9731      & 0.9727 & 0.9472 & 0.9747 & \textbf{0.9802} \\  \midrule
7 $\rightarrow$ 4 & 0.8477 & 0.9207 & 0.9391 & 0.9587 & 0.8245 & 0.8990      & 0.8932 & 0.9274 & 0.9216 & \textbf{0.9723} \\  \midrule
2 $\rightarrow$ 0 & 0.7286 & 0.6855 & 0.6447 & 0.7594 & 0.7390 & 0.7426      & 0.7474 & 0.7835 & 0.7398 & \textbf{0.8058} \\  \midrule
5 $\rightarrow$ 7 & 0.5064 & 0.8987 & 0.9218 & 0.9186 & 0.5028 & 0.8234      & 0.8540 & 0.8200 & 0.8200 & \textbf{0.9378} \\  \midrule
0 $\rightarrow$ 5 & 0.4140 & 0.5714 & 0.5064 & 0.4410 & 0.4138 & 0.4365      & 0.4794 & 0.5444 & 0.5055 & \textbf{0.7899} \\  \midrule
3 $\rightarrow$ 5 & 0.8653 & 0.9701 & 0.9704 & 0.9799 & 0.8828 & 0.9317      & 0.9595 & 0.9750 & 0.9696 & \textbf{0.9824} \\  \midrule
0 $\rightarrow$ 2 & 0.5332 & 0.6743 & 0.7588 & 0.7083 & 0.5723 & 0.6217      & 0.6673 & 0.7239 & 0.6980 & \textbf{0.8028} \\  \midrule
5 $\rightarrow$ 8 & 0.9279 & 0.9796 & 0.9739 & 0.9915 & 0.8242 & 0.9818      & 0.9776 & 0.9874 & 0.9868 & \textbf{0.9927} \\  \midrule
1 $\rightarrow$ 6 & 0.6978 & 0.9047 & 0.9313 & 0.9353 & 0.6160 & 0.8751      & 0.9089 & 0.9144 & 0.8890 & \textbf{0.9415} \\  \midrule
7 $\rightarrow$ 3 & 0.9155 & 0.9338 & 0.9637 & 0.9547 & 0.8912 & 0.9314      & 0.9379 & 0.9575 & 0.9397 & \textbf{0.9688} \\  \bottomrule
\end{tabular}
}
\label{tab:hhar}
\end{table*}

\begin{table}[]
\small
\renewcommand{\arraystretch}{1.3}
\caption{Classification accuracy on the HMDB-UCF video dataset.}
\begin{tabular}{c|c|c|c|c}
\toprule
Method   & Backbone   & U $\rightarrow$ H & H $\rightarrow$ U & Average \\ \midrule
DANN     & ResNet-101 & 75.28  & 76.36  & 75.82   \\
JAN      & ResNet-102 & 74.72  & 76.69  & 75.71   \\
AdaBN    & ResNet-103 & 72.22  & 77.41  & 74.82   \\
MCD      & ResNet-104 & 73.89  & 79.34  & 76.62   \\
TA3N     & ResNet-105 & 78.33  & 81.79  & 80.06   \\
ABG      & ResNet-106 & 79.17  & 85.11  & 82.14   \\
TCoN     & ResNet-107 & 87.22  & 89.14  & 88.18   \\
MA2L-TD  & ResNet-108 & 85     & 86.59  & 85.8    \\ \midrule
DANN     & I3D        & 80.83  & 88.09  & 84.46   \\
ADDA     & I3D        & 79.17  & 88.44  & 83.81   \\
TA3N     & I3D        & 81.38  & 90.54  & 85.96   \\
SAVA     & I3D        & 82.22  & 91.24  & 86.73   \\
CoMix    & I3D        & 86.66  & 93.87  & 90.22   \\
CO2A     & I3D        & 87.78  & 95.79  & 91.79   \\
TranSVAE & I3D        & 87.78  & 98.95  & 93.37   \\ \midrule
Ours      & I3D        & \textbf{89.44}  & \textbf{99.12}  & \textbf{94.28}   \\ \bottomrule
\end{tabular}
\label{classification result}
\end{table}

\begin{figure*}[h]
    \centering
    \label{fig:scatter}
    \includegraphics[width=1.01\textwidth, trim=0cm  11.5cm 0cm 0cm, clip]{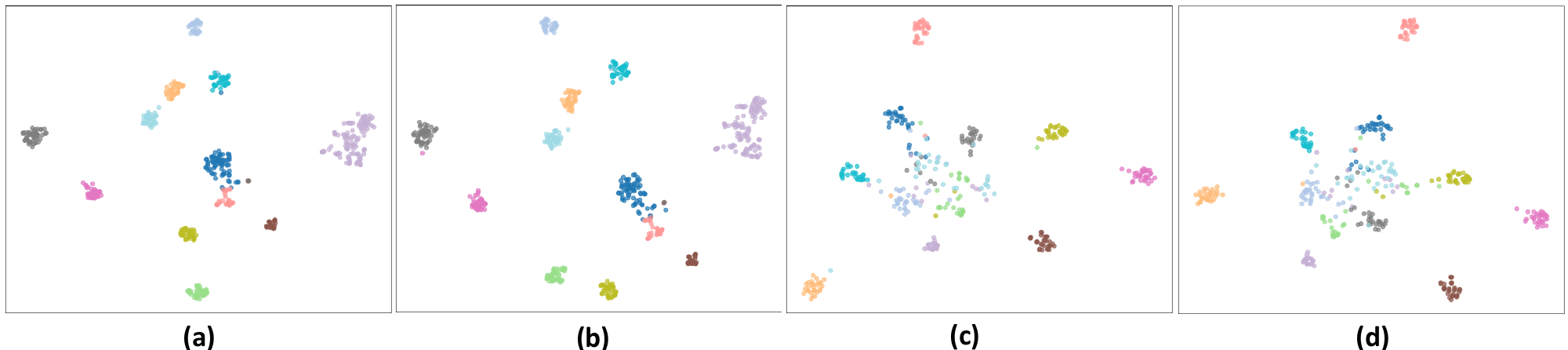}
    \caption{Classification Visualization: (a) and (b) show scatter plots of the latent variables and probability vectors, respectively, after dimensionality reduction using t-SNE for the H$\rightarrow$U scenario. (c) and (d) present the corresponding scatter plots for the U$\rightarrow$H scenario. These visualizations illustrate how the model’s latent representations and probability vectors effectively capture and distinguish between categories.}
\end{figure*}


Table~\ref{classification result} presents a comparative analysis of our method against the baseline on the UCF-HMDB video classification dataset. We notice that methods utilizing the I3D backbone\cite{carreira2017quo} generally outperform those using ResNet-101. Our method consistently exceeds the performance of all prior approaches, achieving an average accuracy of 94.28\%. These results highlight the exceptional performance of our approach even on high-dimensional data.

TranSVAE represents the previous state-of-the-art method, modeling videos as being generated by both dynamic and static latent variables, and has demonstrated remarkable performance in video unsupervised domain adaptation (UDA). Similarly, our approach models the data generation process for time series data—in which videos are a specific instance—from a causal perspective. It assumes that the underlying data generation mechanisms remain consistent across different domains. This causal modeling approach, combined with advanced deep learning architectures, has also delivered impressive results, reinforcing the potential of modeling problems from a causal perspective and leveraging deep learning networks to fit the data generation process.

Additionally, we visualize the classification results for both H$\rightarrow$U and U$\rightarrow$H scenarios. Specifically, we reduced the dimensionality of the model’s output probability vectors or latent variable representations using t-SNE and depicted the results in scatter plots shown in Figure~\ref{fig:scatter}. The visualizations reveal that both the latent variables and probability vectors effectively distinguish between categories.



\begin{figure*}[h]
    \centering
    \includegraphics[width=1.01\textwidth, trim=0.1cm  11.5cm 0.5cm 0cm, clip]{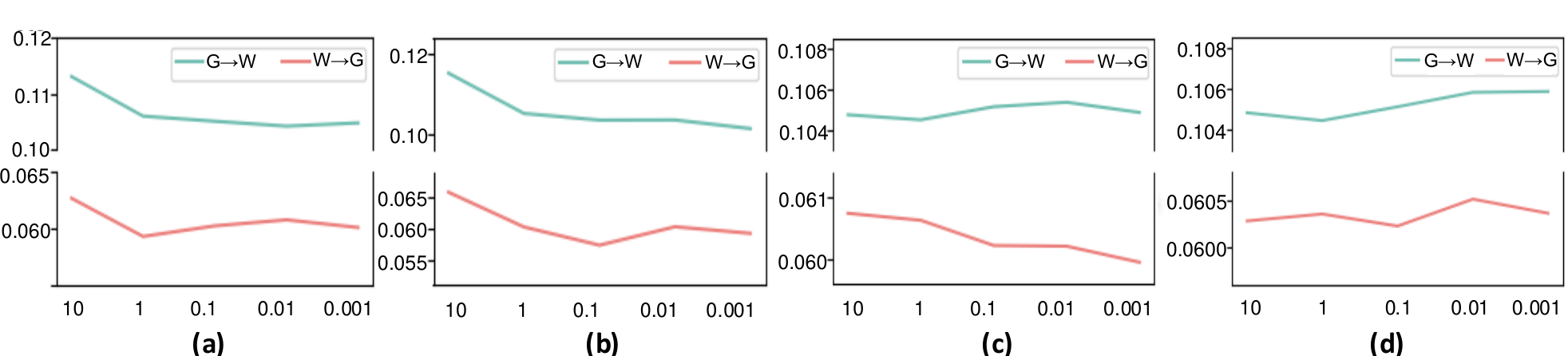}
    \caption{Sensitivity analysis of the parameters \(\alpha\), \(\beta\), \(\gamma\), and \(\delta\) (from Equation~\ref{eqn:overall_loss}) is shown in (a), (b), (c), and (d), respectively. The x-axis represents parameter values, and the y-axis shows the MSE value.}
    \label{fig:sensetivty}
\end{figure*}

\subsection{Ablation Study}
\begin{figure}[h]
    \centering
    \label{fig:ablation}
    \includegraphics[width=1.01\textwidth, trim=0cm  4.2cm 3.2cm 0cm, clip]{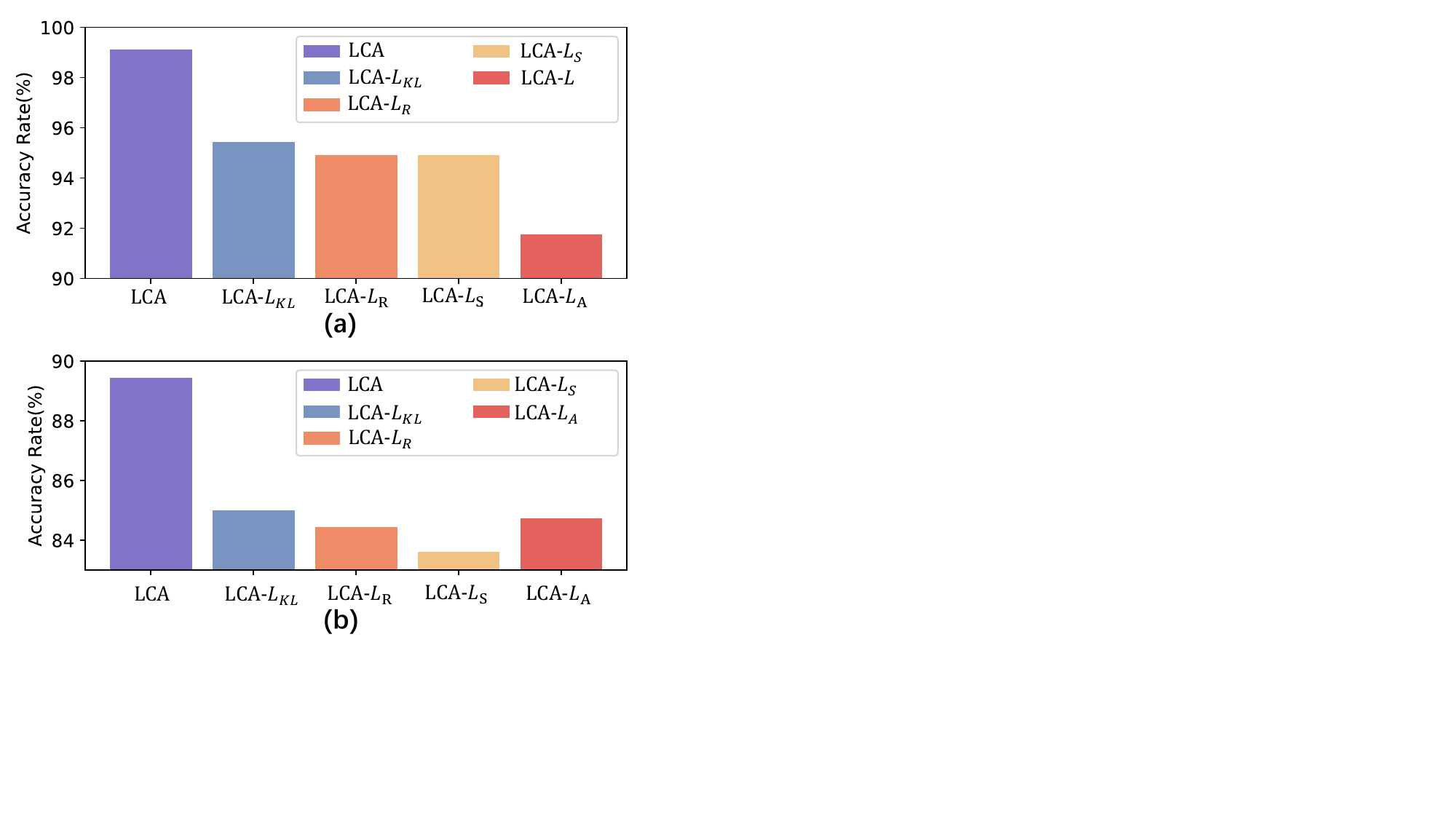}
    \caption{Ablation Results Visualization: (a) presents the results for the H$\rightarrow$U scenario, while (b) shows the results for the U$\rightarrow$H scenario. These visualizations illustrate the impact of different components on the performance of our model across the two scenarios.}
\end{figure}
To rigorously evaluate the contribution of each component in the proposed loss function, we devised four model variants, described as follows.

\begin{itemize}[leftmargin=*]
\setlength{\itemsep}{0pt}
    \item \textbf{LCA-$L_{KL}$}: We remove $L_{KL}$, a critical term in the ELBO.
    \item \textbf{LCA-$L_{R}$}: We remove $L_{R}$, another essential term in the ELBO.
    \item \textbf{LCA-$L_{S}$}: We remove $L_{S}$, which enforces sparsity in the causal structure.
    \item \textbf{LCA-$L_{A}$}: We remove $L_{A}$, which ensures consistency of causal structures across different domains.
\end{itemize}
We conducted the ablation experiments on the high-dimensional video datasets, HMDB, and UCF. In these experiments, we systematically set the weight of each corresponding loss term to zero to observe the performance of the resulting model variants. The results are presented in Figure~\ref{fig:ablation}.
Our findings demonstrate that each loss component plays a crucial role in enhancing the model's performance. Specifically, $\mathcal{L}_{KL}$ and $\mathcal{L}_{R}$ are vital elements of the ELBO, and maximizing them enables the model to accurately learn the joint distribution of the data. From a theoretical standpoint, $\mathcal{L}_{KL}$ ensures that the learned distribution of the latent variables closely approximates the true distribution, while $\mathcal{L}_{R}$ ensures that the latent variables capture as much information from the observed variables as possible.
The term $\mathcal{L}_{S}$ enforces the sparsity of relationships between adjacent latent variables, a requirement rooted in identifiability theory. This constraint aligns with real-world scenarios, where causal relationships are typically sparse, and it helps prevent the model from learning redundant associations. Lastly, $\mathcal{L}_{A}$ aligns the causal mechanisms in the target domain with those in the source domain. The substantial improvement in performance observed in the experiments validates the importance of enforcing consistency in causal mechanisms across domains, thereby further affirming the soundness of our approach.

\subsection{Sensitivity Analysis}
We performed a sensitivity analysis on the parameters \(\alpha\), \(\beta\), \(\gamma\), and \(\delta\) in the loss function (Equation~\ref{eqn:overall_loss}), focusing on the transfer directions from G to W and W to G in the Human Motion dataset. We set several values for each parameter, ranging from 10 to 0.001 and decreasing by a factor of 10 each time. The experimental results are shown in Figure \ref{fig:sensetivty}. 
Notably, as the parameters decrease, the MSE exhibits a trend of initially decreasing and then increasing, which matches our expectations: when the parameters are large, the constraints on the corresponding modules are strongest, causing the model to over-focus on certain areas and leading to a decline in performance. Conversely, when these parameters are small, the corresponding modules receive insufficient attention, resulting in suboptimal performance.

\section{Conclusion}\label{conclusion}
In this paper, we propose a Latent Causality Alignment (LCA) model for time series domain adaptation. By identifying low-dimensional latent variables and reconstructing their causal structures with sparsity constraints, we effectively transfer domain knowledge through latent causal mechanism alignment. The proposed method not only addresses the challenges of modeling latent causal mechanisms in high-dimensional time series data but also guarantees the uniqueness of latent causal structures, providing improved performance on domain-adaptive time series classification and forecasting tasks. However, our method assumes sufficient changes in historical information for latent variable identification. Exploring how to relax this assumption and extend our framework to more complex scenarios would be a promising direction for future work.


\bibliographystyle{IEEEtran}
\bibliography{ref}

\include{supplementary}

\ifCLASSOPTIONcaptionsoff
  \newpage
\fi

\clearpage

\end{document}

%% file: math_commands.tex
\usepackage{amsmath,amsfonts,bm}

\def\eqref#1{equation~\ref{#1}}

\def\1{\bm{1}}

\def\rvg{{\mathbf{g}}}

\def\rvu{{\mathbf{i}}}

\def\rvu{{\mathbf{u}}}

\def\rvx{{\mathbf{x}}}

\def\rvz{{\mathbf{z}}}

\def\mA{{\bm{A}}}

\def\mJ{{\bm{J}}}

\DeclareMathAlphabet{\mathsfit}{\encodingdefault}{\sfdefault}{m}{sl}
\SetMathAlphabet{\mathsfit}{bold}{\encodingdefault}{\sfdefault}{bx}{n}



%% file: supplementary.tex
\onecolumn 

\appendix
\begin{lemma} \label{lemma: 1}
Suppose there exists invertible function $\hat{\mathbf{g}}$ that maps $\mathbf{x}_t$ to $\hat{\mathbf{z}}_t$, i.e.
\begin{equation}
\label{eq: hatz_xt_invertible_function}
    \hat{\mathbf{z}}_t=\hat{\mathbf{g}}_(\mathbf{x}_t)
\end{equation}
such that the components of $\hat{\mathbf{z}}_t$ are mutually independent conditional on $\hat{\mathbf{z}}_{t-1}$ .Let
\begin{equation}
\label{eq: second_third_derivatives}
 \begin{split}
    \mathbf{v}_{t,k} =
    \Big(\frac{\partial^{2}\log p(z_{t,k}|\mathbf{z}_{t-1}) }{\partial z_{t,k}\partial z_{t-1,1}},\frac{\partial^{2}\log p(z_{t,k}|\mathbf{z}_{t-1})}{\partial z_{t,k}\partial z_{t-1,2}},...,
        \frac{\partial^{2}\log p(z_{t,k}|\mathbf{z}_{t-1})}{\partial z_{t,k}\partial z_{t-1,n}}\Big)^{\mathsf{T}},\\
    \mathring{\mathbf{v}}_{t,k}=
\Big(\frac{\partial^{3}\log p(z_{t,k}|\mathbf{z}_{t-1})}{\partial z_{t,k}^{2}\partial z_{t-1,1}},\frac{\partial^{3}\log p(z_{t,k}|\mathbf{z}_{t-1})}{\partial z_{t,k}^{2}\partial z_{t-1,2}},...,
    \frac{\partial^{3}\log p(z_{t,k}|\mathbf{z}_{t-1})}{\partial z_{t,k}^{2}\partial z_{t-1,n}}\Big)^{\mathsf{T}}.
    \end{split}
\end{equation}
If for each value of $\mathbf{z}_t,\mathbf{v}_{t,1},\mathring{\mathbf{v}}_{t,1},\mathbf{v}_{t,2},\mathring{\mathbf{v}}_{t,2},...,,\mathbf{v}_{t,n},\mathring{\mathbf{v}}_{t,n}$, as 2n vector function in $z_{t-1,1},z_{t-1,2},...,z_{t-1,n}$, are linearly independent, then $\mathbf{z}_t$ must be an invertible, component-wise transformation of a permuted version of $\hat{\mathbf{z}}_t$.
\end{lemma}

\begin{proof}
    Combining Equation (\ref{eq: hatz_xt_invertible_function}) and Equation (\ref{eq: second_third_derivatives}) gives $\mathbf{z}_t=\mathbf{g}^{-1}(\hat{\mathbf{g}}^{-1}(\hat{\mathbf{z}}_t))$, where $\mathbf{h}=\mathbf{g}^{-1}\circ \hat{\mathbf{g}}^{-1}$.
 Since both $\hat{\mathbf{g}}$ and $\mathbf{g}$ are invertible, $\mathbf{h}$ is invertible. Let $\mathbf{H}_t$ be the Jacobian matrix of the transformation $h(\hat{\mathbf{z}}_t)$, and denote by $\mathbf{H}_{tki}$ its $(k,i)$th entry.

 First, it is straightforward to see that if the components of $\hat{\mathbf{z}}_t$ are mutually independent conditional on $\hat{\mathbf{z}}_{t-1}$ , then for any $i\neq j$, $\hat{z}_{t,i}$ and $\hat{z}_{t,k}$ are conditionally independent given $\hat{\mathbf{z}}_{t-1} \cup (\hat{\mathbf{z}}_t \setminus \{\hat{z}_{t,i},\hat{z}_{t,j}\}) $ .Mutual independence of the components of $\hat{\mathbf{z}}_t$ conditional of $\hat{\mathbf{z}}_{t-1}$ implies that $\hat{z}_{t,i}$ is independent from $\hat{\mathbf{z}}_t \setminus \{\hat{z}_{t,i}, \hat{z}_{t,j} \}$ conditional on $\hat{\mathbf{z}}_{t-1}$ , i.e.,
 $$p(\hat{z}_{t,i} \,|\, \hat{\mathbf{z}}_{t-1}) = p(\hat{z}_{t,i} \,|\, \hat{\mathbf{z}}_{t-1} \cup (\hat{\mathbf{z}}_t \setminus \{\hat{z}_{t,i},\hat{z}_{t,j}\})).$$
At the same time, it also implies $\hat{z}_{t,i}$ is independent from $\hat{\mathbf{z}}_{t} \setminus \{\hat{z}_{t,i}\}$ conditional on $\hat{\mathbf{z}}_{t-1}$, i.e.,
    $$p(\hat{z}_{t,i} \,|\, \hat{\mathbf{z}}_{t-1}) = p(\hat{\mathbf{z}}_{t,i} \,|\, \hat{\mathbf{z}}_{t-1} \cup (\hat{\mathbf{z}}_t \setminus \{\hat{z}_{t,i}\})).$$
Combining the above two equations gives $ p(\hat{z}_{t,i} \,|\, \hat{\mathbf{z}}_{t-1} \cup (\hat{\mathbf{z}}_t \setminus  \{\hat{z}_{t,i}\}))=
p(\hat{z}_{t,i} \,|\, \hat{\mathbf{z}}_{t-1} \cup (\hat{\mathbf{z}}_t  \setminus  \{\hat{z}_{t,i},\hat{z}_{t,j}\}))$, i.e., for $i\neq j$ , $\hat{z}_{t,i}$ and $\hat{z}_{t,j}$ are conditionally independent given $\hat{\mathbf{z}}_{t-1} \cup (\hat{\mathbf{z}}_t\setminus \{\hat{z}_{t,i}, \hat{z}_{t,j}\})$.
We then make use of the fact that if $\hat{z}_{t,i}$ and $\hat{z}_{t,j}$ are conditionally independent given $\hat{\mathbf{z}}_{t-1} \cup (\hat{\mathbf{z}}_t\setminus \{\hat{z}_{t,i}, \hat{z}_{t,j}\})$ , then
    $$\frac{\partial^2\log p(\hat{\mathbf{z}}_t,\hat{\mathbf{z}}_{t-1})} {\partial \hat{z}_{t,i} \partial \hat{z}_{t,j}} = 0,$$
assuming the cross second-order derivative exists . Since $p(\hat{\mathbf{z}}_t,\hat{\mathbf{z}}_{t-1})=p(\hat{\mathbf{z}}_t \,|\, \hat{\mathbf{z}}_{t-1})p(\hat{\mathbf{z}}_{t-1})$ while $p(\hat{\mathbf{z}}_{t-1})$ does not involve $\hat{z}_{t,i}$ or $\hat{z}_{t,j}$ , the above equality is equivalent to
\begin{equation} \label{Eq:iszero}
    \frac{\partial^2\log p(\hat{\mathbf{z}}_t \,|\, \hat{\mathbf{z}}_{t-1})} {\partial \hat{z}_{t,i} \partial \hat{z}_{t,j}} = 0
\end{equation}
The Jacobian matrix of the mapping from $(\mathbf{x}_{t-1}, \hat{\mathbf{z}}_t)$ to $(\mathbf{x}_{t-1}, \mathbf{z}_t)$ is $\begin{bmatrix}\mathbf{I} & \mathbf{0} \\ * & \mathbf{H}_t \end{bmatrix}$, where $*$ stands for a matrix, and the (absolute value of the) determinant of this Jacobian matrix is $|\mathbf{H}_t|$. Therefore $p(\hat{\mathbf{z}}_t, \mathbf{x}_{t-1}) = p({\mathbf{z}}_t, \mathbf{x}_{t-1})\cdot |\mathbf{H}_t|$. Dividing both sides of this equation by $p(\mathbf{x}_{t-1})$ gives 
\begin{equation} \label{Eq:J_trans}
 p(\hat{\mathbf{z}}_t \,|\, \mathbf{x}_{t-1}) = p({\mathbf{z}}_t \,|\, \mathbf{x}_{t-1}) \cdot |\mathbf{H}_t|. 
 \end{equation}
  Since $p({\mathbf{z}}_t \,|\, {\mathbf{z}}_{t-1}) = p({\mathbf{z}}_t \,|\, \mathbf{g}({\mathbf{z}}_{t-1})) = p({\mathbf{z}}_t \,|\, {\mathbf{x}}_{t-1})$ and similarly   $p(\hat{\mathbf{z}}_t \,|\, \hat{\mathbf{z}}_{t-1}) = p(\hat{\mathbf{z}}_t \,|\, {\mathbf{x}}_{t-1})$, Equation (\ref{Eq:J_trans}) tells us
 \begin{equation}
     \log p(\hat{\mathbf{z}}_t \,|\, \hat{\mathbf{z}}_{t-1}) = \log p({\mathbf{z}}_t \,|\, {\mathbf{z}}_{t-1}) + \log |\mathbf{H}_t| = \sum_{k=1}^n \log p(z_{t,k}|\mathbf{z}_{t-1})  + \log |\mathbf{H}_t|.
 \end{equation}
  Its partial derivative w.r.t. $\hat{z}_{t,i}$ is
  \begin{flalign} \nonumber
  \frac{\partial \log p(\hat{\mathbf{z}}_t \,|\, \hat{\mathbf{z}}_{t-1})}{\partial \hat{z}_{t,i}} &=  \sum_{k=1}^n \frac{\partial \log p(z_{t,k}|\mathbf{z}_{t-1})  }{\partial z_{t,k}} \cdot \frac{\partial z_{t,k}}{\partial \hat{z}_{t,i}} - \frac{\partial \log |\mathbf{H}_t|}{\partial \hat{z}_{t,i}} \\ \nonumber
  &= \sum_{k=1}^n \frac{\partial \log p(z_{t,k}|\mathbf{z}_{t-1}) }{\partial z_{t,k}} \cdot \mathbf{H}_{tki} - \frac{\partial \log |\mathbf{H}_t|}{\partial \hat{z}_{t,i}}.
 \end{flalign}
   Its second-order cross derivative is
   \begin{flalign} \label{Eq:cross}
  \frac{\partial^2 \log p(\hat{\mathbf{z}}_t \,|\, \hat{\mathbf{z}}_{t-1})}{\partial \hat{z}_{t,i} \partial \hat{z}_{t,j}}
  &= \sum_{k=1}^n \Big( \frac{\partial^2 \log p(z_{t,k}|\mathbf{z}_{t-1}) }{\partial z_{t,k}^2 } \cdot \mathbf{H}_{tki}\mathbf{H}_{tkj} + \frac{\partial \log p(z_{t,k}|\mathbf{z}_{t-1}) }{\partial z_{t,k}} \cdot \frac{\partial \mathbf{H}_{tki}}{\partial \hat{z}_{t,j}} \Big)- \frac{\partial^2 \log |\mathbf{H}_t|}{\partial \hat{z}_{t,i} \partial \hat{z}_{t,j}}.
 \end{flalign}
The above quantity is always 0 according to Equation (\ref{Eq:iszero}). Therefore, for each $l=1,2,...,n$ and each value $z_{t-1,l}$,  its partial derivative w.r.t.
 $z_{t-1,l}$ is always 0. That is,
  \begin{flalign}\label{eq:lind-ap}
  \frac{\partial^3 \log p(\hat{\mathbf{z}}_t \,|\, \hat{\mathbf{z}}_{t-1})}{\partial \hat{z}_{t,i} \partial \hat{z}_{t,j} \partial z_{t-1,l}}
  &= \sum_{k=1}^n \Big( \frac{\partial^3 \log p(z_{t,k}|\mathbf{z}_{t-1}) }{\partial z_{t,k}^2 \partial z_{t-1,l}} \cdot \mathbf{H}_{tki}\mathbf{H}_{tkj} + \frac{ \partial^2 \log p(z_{t,k}|\mathbf{z}_{t-1}) }{\partial z_{t,k} \partial z_{t-1,l}}  \cdot \frac{\partial \mathbf{H}_{tki}}{\partial \hat{z}_{t,j} } \Big) \equiv 0,
 \end{flalign}
  where we have made use of the fact that entries of $\mathbf{H}_t$ do not depend on $z_{t-1,l}$. 

If for any value of  $\mathbf{z}_t,\mathbf{v}_{t,1},\mathring{\mathbf{v}}_{t,1},\mathbf{v}_{t,2},\mathring{\mathbf{v}}_{t,2},...,,\mathbf{v}_{t,n},\mathring{\mathbf{v}}_{t,n}$ are linearly independent, to make the above equation hold true, one has to set $\mathbf{H}_{tki}\mathbf{H}_{tkj} = 0$ or $i\neq j$. That is, in each row of $\mathbf{H}_t$ there is only one non-zero entry. Since $h$ is invertible, then $\mathbf{z}_{t}$ must be an invertible, component-wise transformation of a permuted version of $\hat{\mathbf{z}}_t$.
\end{proof}

\section{Extension to Multiple Lags and Sequence Lengths}
\label{app: Extension to Multiple Lags and Sequence Lengths}
 For the sake of simplicity, we consider only one special case with $\tau=1$ and $L=2$ in Lemma \ref{lemma: 1}. Our identifiability proof can actually be applied for arbitrary lags directly. For instance, in the stationary case in Equation (\ref{eq: second_third_derivatives}), one can simply let
 \begin{equation}
\label{eq: any_tau_derivatives}
 \begin{split}
    \mathbf{v}_{t,k} =
    \Big(\frac{\partial^{2}\log p(z_{t,k}|\mathbf{z}_{Hx}) }{\partial z_{t,k}\partial z_{t-\tau,1}},\frac{\partial^{2}\log p(z_{t,k}|\mathbf{z}_{Hx})}{\partial z_{t,k}\partial z_{t-\tau,2}},...,
        \frac{\partial^{2}\log p(z_{t,k}|\mathbf{z}_{Hx})}{\partial z_{t,k}\partial z_{t-\tau,n}}\Big)^{\mathsf{T}},\\
    \mathring{\mathbf{v}}_{t,k}=
\Big(\frac{\partial^{3}\log p(z_{t,k}|\mathbf{z}_{Hx})}{\partial z_{t,k}^{2}\partial z_{t-\tau,1}},\frac{\partial^{3}\log p(z_{t,k}|\mathbf{z}_{Hx})}{\partial z_{t,k}^{2}\partial z_{t-\tau,2}},...,
    \frac{\partial^{3}\log p(z_{t,k}|\mathbf{z}_{Hx})}{\partial z_{t,k}^{2}\partial z_{t-\tau,n}}\Big)^{\mathsf{T}}.
    \end{split}
\end{equation}
, where  $\mathbf{z}_{Hx}$ denotes the lagged latent variables up to maximum time lag $L$. We take derivatives with regard to $z_{ t-\tau,1}, …, z_{ t-\tau,n}$, which can be any latent temporal variables at lag $\tau$, instead of $z_{ t-1,1}, …, z_{t-1,n}$. If there exists one $\tau$ (out of the $L$ lags) that satisfies the condition, then the stationary latent processes are identifiable. 


\section{Granger non-causality of Latent Variables}
\label{app: Granger non-causality of Latent Variables}

In this section, we will theoretically analyze the identification of the proposed model. To achieve this, we begin with the definition of the Granger non-causality of latent variables as follows. 
\begin{definition}
\label{def: Granger non-causality of Latent Variables}
      Suppose that the object sequences $Z$ are generated according to Equation (1) and Equation (2), we can determine the Granger non-causality of the object $\mathbf{z}_{t-1,i}$ with respect to object $\mathbf{z}_{t,j}$ as follows: For all $\mathbf{z}_{t-1:t-\tau,1},\mathbf{z}_{t-1:t-\tau,2},\cdots,\mathbf{z }_{t-1:t-\tau,n}$,and the same variable with different values $\mathbf{z}_{t-1:t-\tau,i} \neq \mathbf{z}^{'}_{t-1:t-\tau,i}$,if the following condition holds:
\begin{equation}
\phi_j(\mathbf{z}_{t - 1:t - \tau,1}, \cdots, \mathbf{z}_{t - 1:t - \tau,i}, \cdots, \mathbf{z}_{t - 1:t - \tau,n}) = \phi_j(\mathbf{z}_{t - 1:t - \tau,1}, \cdots, \mathbf{z}_{t - 1:t - \tau,i}', \cdots, \mathbf{z}_{t - 1:t - \tau,n})
\end{equation}
that is, $\mathbf{x}_{t,j}$ is invariant to $\mathbf{x}_{t-1:t-\tau,i}$ with $\phi_j$.
\end{definition}

\begin{prop}
    Suppose that the estimated function $f_j$ well models the relationship between $\mathbf{z}_{t}$ and $\mathbf{z}_{t-1}$.Given the ground truth Granger Causal structure $\mathcal{G}=(V, E_V)$ with the maximum lag of 1, where $V$ and $E_V$ denote the nodes and edges, respectively. When the data are generated by Equation (1) and Equation (2), then $\mathbf{z}_{t-1,i} \rightarrow \mathbf{z}_{t,j} \notin E_V$ if and only if $\frac{\partial \mathbf{z}_{t,j}}{\partial  \mathbf{z}_{t-1,i}}=0$.
\end{prop}

\begin{proof}
    $\Rightarrow$If $\mathbf{z}_{t-1,i} \rightarrow \mathbf{z}_{t,j} \notin E_V$ (there is no Granger Causality between $\mathbf{z}_{t-1,i}$ and $\mathbf{z}_{t,j}$ in the ground truth process), then according to Definition \ref{def: Granger non-causality of Latent Variables},there must be $\mathbf{z}_{t,j}\neq \mathbf{z}'_{t,j}$ for all the different values of $\mathbf{z}_{t-1,i}$. If  $\frac{\partial \mathbf{z}_{t,j}}{\partial  \mathbf{z}_{t-1,i}}=0$, then when the input values are different, i.e. $\mathbf{z}_{t-1,i} \neq \mathbf{z}'_{t-1,i}$,the outputs of $\mathbf{f}_j$ are also different, i.e., $\mathbf{z}_{t,j}\neq \mathbf{z}'_{t,j}$ , which results in contradictions.\\
    $\Leftarrow$: Suppose $\mathbf{z}_{t-1,i} \rightarrow \mathbf{z}_{t,j} \in E_V$ (there is Granger Causality between $\mathbf{z}_{t-1,i}$ and $\mathbf{z}_{t,j}$ in the ground truth process), for any $\mathbf{z}_{t,j}$ and $\mathbf{z}_{t,j}'$ in Definition \ref{def: Granger non-causality of Latent Variables}, if $\mathbf{z}_{t,j} \neq \mathbf{z}'_{t,j}$, there must be different input $\mathbf{z}_{t-1,i} \neq \mathbf{z}'_{t-1,i}$. If $\frac{\partial \mathbf{z}_{t,j}}{\partial  \mathbf{z}_{t-1,i}}=0$, then when $\mathbf{z}_{t,j} \neq \mathbf{z}_{t,j}'$, there might be $\mathbf{z}_{t-1,i} = \mathbf{z}_{t-1,i}'$, which results in contradictions.
    
\end{proof}

\section{Prior Likelihood Derivation} \label{app:prior}

 We first consider the prior of $\ln p(\rvz_{1:T})$. We start with an illustrative example of latent causal processes with two time-delay latent variables, i.e. $\rvz_t=[z_{t,1}, z_{t,2}]$ with maximum time lag $\tau=1$, i.e., $z_{t,i}=f_i(\rvz_{t-1}, \epsilon_{t,i})$ with mutually independent noises. Then we write this latent process as a transformation map $\mathbf{f}$ (note that we overload the notation $f$ for transition functions and for the transformation map):
    \begin{equation}
    \small
\begin{gathered}\nonumber
    \begin{bmatrix}
    \begin{array}{c}
        z_{t-1,1} \\ 
        z_{t-1,2} \\
        z_{t,1}   \\
        z_{t,2}
    \end{array}
    \end{bmatrix}=\mathbf{f}\left(
    \begin{bmatrix}
    \begin{array}{c}
        z_{t-1,1} \\ 
        z_{t-1,2} \\
        \epsilon_{t,1}   \\
        \epsilon_{t,2}
    \end{array}
    \end{bmatrix}\right).
\end{gathered}
\end{equation}
By applying the change of variables formula to the map $\mathbf{f}$, we can evaluate the joint distribution of the latent variables $p(z_{t-1,1},z_{t-1,2},z_{t,1}, z_{t,2})$ as 
\begin{equation}
\small
\label{equ:p1}
    p(z_{t-1,1},z_{t-1,2},z_{t,1}, z_{t,2})=\frac{p(z_{t-1,1}, z_{t-1,2}, \epsilon_{t,1}, \epsilon_{t,2})}{|\text{det }\mathbf{J}_{\mathbf{f}}|},
\end{equation}
where $\mathbf{J}_{\mathbf{f}}$ is the Jacobian matrix of the map $\mathbf{f}$, which is naturally a low-triangular matrix:
\begin{equation}
\small
\begin{gathered}\nonumber
    \mathbf{J}_{\mathbf{f}}=\begin{bmatrix}
    \begin{array}{cccc}
        1 & 0 & 0 & 0 \\
        0 & 1 & 0 & 0 \\
        \frac{\partial z_{t,1}}{\partial z_{t-1,1}} & \frac{\partial z_{t,1}}{\partial z_{t-1,2}} & 
        \frac{\partial z_{t,1}}{\partial \epsilon_{t,1}} & 0 \\
        \frac{\partial z_{t,2}}{\partial z_{t-1, 1}} &\frac{\partial z_{t,2}}{\partial z_{t-1,2}} & 0 & \frac{\partial z_{t,2}}{\partial \epsilon_{t,2}}
    \end{array}
    \end{bmatrix}.
\end{gathered}
\end{equation}
Given that this Jacobian is triangular, we can efficiently compute its determinant as $\prod_i \frac{\partial z_{t,i}}{\epsilon_{t,i}}$. Furthermore, because the noise terms are mutually independent, and hence $\epsilon_{t,i} \perp \epsilon_{t,j}$ for $j\neq i$ and $\epsilon_{t} \perp \rvz_{t-1}$, so we can with the RHS of Equation (\ref{equ:p1}) as follows
\begin{equation}
\label{equ:p2}
\begin{array}{rl}
    p(z_{t-1,1}, z_{t-1,2}, z_{t,1}, z_{t,2}) & = p(z_{t-1,1}, z_{t-1,2}) \times \frac{p(\epsilon_{t,1}, \epsilon_{t,2})}{|\mathbf{J}_{\mathbf{f}}|} \\
    & = p(z_{t-1,1}, z_{t-1,2}) \times \frac{\prod_i p(\epsilon_{t,i})}{|\mathbf{J}_{\mathbf{f}}|} \\
\end{array}
\end{equation}

Finally, we generalize this example and derive the prior likelihood below. Let $\{r_i\}_{i=1,2,3,\cdots}$ be a set of learned inverse transition functions that take the estimated latent causal variables, and output the noise terms, i.e., $\hat{\epsilon}_{t,i}=r_i(\hat{z}_{t,i}, \{ \hat{\rvz}_{t-\tau}\})$. Then we design a transformation $\mathbf{A}\rightarrow \mathbf{B}$ with low-triangular Jacobian as follows:
\begin{equation}
\small
\begin{gathered}
    \underbrace{[\hat{\rvz}_{t-\tau},\cdots,{\hat{\rvz}}_{t-1},{\hat{\rvz}}_{t}]^{\top}}_{\mathbf{A}} \text{  mapped to  } \underbrace{[{\hat{\rvz}}_{t-\tau},\cdots,{\hat{\rvz}}_{t-1},{\hat{\epsilon}}_{t,i}]^{\top}}_{\mathbf{B}}, \text{ with } \mathbf{J}_{\mathbf{A}\rightarrow\mathbf{B}}=
    \begin{bmatrix}
    \begin{array}{cc}
        \mathbb{I}_{n\times \tau} & 0\\
                    * & \text{diag}\left(\frac{\partial r_{i,j}}{\partial {\hat{z}}_{t,j}}\right)
    \end{array}
    \end{bmatrix}.
\end{gathered}
\end{equation}
Similar to Equation (\ref{equ:p2}), we can obtain the joint distribution of the estimated dynamics subspace as:
\begin{equation}
    \log p(\mathbf{A})=\underbrace{\log p({\hat{\rvz}}_{t-\tau},\cdots, {\hat{\rvz}}_{t-1}) + \sum^{n}_{i=1}\log p({\hat{\epsilon}}_{t,i})}_{\text{Because of mutually independent noise assumption}}+\log (|\text{det}(\mathbf{J}_{\mathbf{A}\rightarrow\mathbf{B}})|)
\end{equation}
Finally, we have:
\begin{equation}
\small
    \log p({\hat{\rvz}}_t|\{{\hat{\rvz}}_{t-\tau}\})=\sum_{i=1}^{n} \log {p({\hat{\epsilon}_{t,i}})} + \sum_{i=1}^{n}\log |\frac{\partial r_i}{\partial {\hat{z}}_{t,i}}|
\end{equation} 
Since the prior of $p(\hat{\rvz}_{1:T})=p(\hat{z}_{1:\tau})\prod_{t=\tau+1}^{T} p(\hat{\rvz}_{t}| \{ \hat{\rvz}_{t-\tau}\})$ with the assumption of first-order Markov assumption, we can estimate $\log p(\hat{\rvz}_{1:T})$ as follows:
\begin{equation}
\small
    \log p(\hat{\rvz}_{1:T})=\log {p(\hat{z}_{1:\tau})} \prod_{t=\tau+1}^T\left(\sum_{i=1}^n \log p(\hat{\epsilon}_{t,i}) + \sum_{i=1}^n \log |\frac{\partial r_i}{\partial z_{t,i}}|  \right),
\end{equation}
Where the noise \( p(\hat{\epsilon}_{\tau,i}) \) and the initial latent variables \( p(\hat{\rvz})_{1:\tau} \) are both assumed to follow Gaussian distributions. When $\tau=1$, $\log p(\rvz_{1:T})=\log {p(\hat{z}_{1})}\prod_{t=2}^T\left(\sum_{i=1}^n \log p(\hat{\epsilon}_{t,i}) + \sum_{i=1}^n \log |\frac{\partial r_i}{\partial z_{t,i}}|  \right)$. In a similar manner, the distribution \( p(\hat{\rvz}_{t+1:T}|\hat{\rvz}_{1:t}) \) can be estimated using analogous methods.

\section{More Experiment Results}

This is more experiment results

\begin{figure*}[h]
    \centering
    \includegraphics[width=1.01\textwidth, trim=0cm  12cm 0cm 0cm, clip]{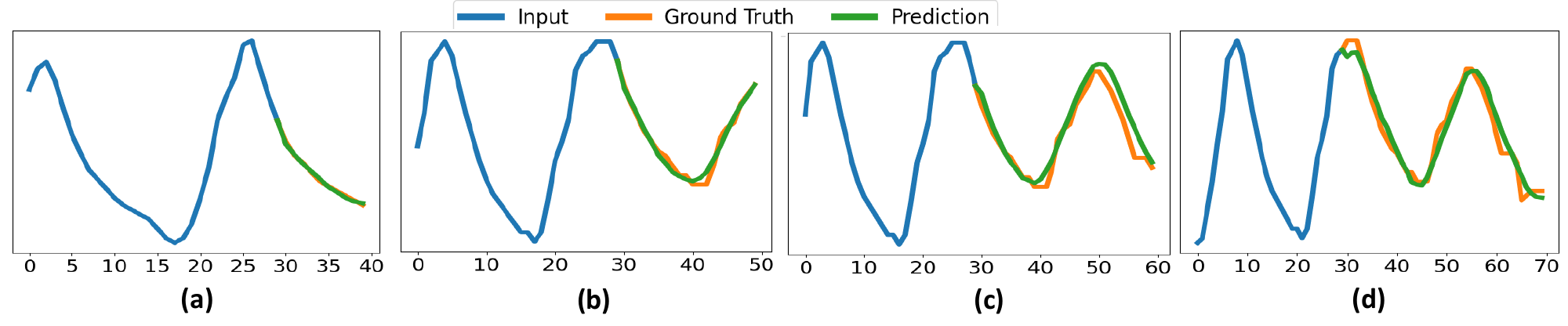}
    \caption{Visualization of prediction results across varying forecast lengths for the transition from domain 1 to domain 2 in the ETT dataset. Subfigures (a), (b), (c), and (d) represent forecast lengths of 10, 20, 30, and 40, respectively.}
    \label{fig:diff_len_mase}
\end{figure*}

\begin{table*}[htbp]
\caption{The MAE and MSE on the Electricity Load Diagrams dataset, where R-COAT, iTRANS, and TMixer are the abbreviation of RAINCOAT, iTransformer, and
TimeMixer, respectively.}
\label{tab:eld}
\setlength{\tabcolsep}{3mm}
\resizebox{\textwidth}{!}{%
\begin{tabular}{@{}c|c|ccccccccccc@{}}
\toprule
Metric                & Task              & SASA            & GCA    & DAF    & CLUDA  & R-COAT & AdvSKM & iTRANS & TMixer & TSLANet & SegRNN & Ours            \\ \midrule
\multirow{12}{*}{\rotatebox{90}{MSE}} & 1 $\rightarrow$ 2 & 0.1781          & 0.2919 & 0.2142 & 0.2412 & 0.2472 & 0.2143 & 0.1778 & 0.2066 & 0.2551  & 0.2675 & \textbf{0.1452} \\
                      & 1 $\rightarrow$ 3 & 0.1544          & 0.2810 & 0.1772 & 0.2133 & 0.2251 & 0.1879 & 0.1723 & 0.2064 & 0.2464  & 0.2356 & \textbf{0.1410} \\
                      & 1 $\rightarrow$ 4 & 0.1855          & 0.3421 & 0.2043 & 0.2689 & 0.2804 & 0.2205 & 0.2121 & 0.3375 & 0.2763  & 0.3018 & \textbf{0.1586} \\ \cmidrule{2-13} 
                      & 2 $\rightarrow$ 1 & 0.1216          & 0.2062 & 0.1426 & 0.1503 & 0.1523 & 0.1217 & 0.1758 & 0.2742 & 0.2401  & 0.2372 & \textbf{0.1171} \\
                      & 2 $\rightarrow$ 3 & 0.1315          & 0.2097 & 0.1527 & 0.1643 & 0.1625 & 0.1372 & 0.1459 & 0.2249 & 0.2172  & 0.2144 & \textbf{0.1129} \\
                      & 2 $\rightarrow$ 4 & 0.1892          & 0.3018 & 0.2480 & 0.2528 & 0.2285 & 0.2157 & 0.2084 & 0.3372 & 0.3161  & 0.2903 & \textbf{0.1587} \\ \cmidrule{2-13} 
                      & 3 $\rightarrow$ 1 & \textbf{0.1360} & 0.2514 & 0.1550 & 0.1560 & 0.1922 & 0.1473 & 0.1927 & 0.2840 & 0.2612  & 0.2512 & \textbf{0.1384} \\
                      & 3 $\rightarrow$ 2 & 0.1181          & 0.2262 & 0.1305 & 0.1468 & 0.1625 & 0.1190 & 0.1577 & 0.2371 & 0.2281  & 0.2250 & \textbf{0.1113} \\
                      & 3 $\rightarrow$ 4 & 0.1874          & 0.3277 & 0.2166 & 0.2835 & 0.2868 & 0.2505 & 0.2327 & 0.3328 & 0.3054  & 0.2928 & \textbf{0.1839} \\ \cmidrule{2-13} 
                      & 4 $\rightarrow$ 1 & 0.1203          & 0.2496 & 0.1277 & 0.1416 & 0.1589 & 0.1206 & 0.1791 & 0.2380 & 0.2449  & 0.2308 & \textbf{0.1063} \\
                      & 4 $\rightarrow$ 2 & 0.1442          & 0.2591 & 0.1636 & 0.1851 & 0.1829 & 0.1600 & 0.1655 & 0.2140 & 0.2493  & 0.2340 & \textbf{0.1176} \\
                      & 4 $\rightarrow$ 3 & 0.1271          & 0.2153 & 0.1372 & 0.1670 & 0.1734 & 0.1386 & 0.1558 & 0.2779 & 0.2435  & 0.2155 & \textbf{0.1100} \\ \midrule
\multirow{12}{*}{\rotatebox{90}{MAE}} & 1 $\rightarrow$ 2 & 0.3103          & 0.4131 & 0.3508 & 0.3682 & 0.3866 & 0.3494 & 0.3116 & 0.3264 & 0.3794  & 0.3873 & \textbf{0.2821} \\
                      & 1 $\rightarrow$ 3 & 0.2959          & 0.4107 & 0.3228 & 0.3531 & 0.3717 & 0.3343 & 0.3127 & 0.3305 & 0.3760  & 0.3606 & \textbf{0.2816} \\
                      & 1 $\rightarrow$ 4 & 0.3270          & 0.4548 & 0.3498 & 0.3999 & 0.4167 & 0.3667 & 0.3511 & 0.4419 & 0.3972  & 0.4147 & \textbf{0.3007} \\ \cmidrule{2-13} 
                      & 2 $\rightarrow$ 1 & 0.2593          & 0.3396 & 0.2842 & 0.2946 & 0.2954 & 0.2595 & 0.3065 & 0.3883 & 0.3638  & 0.3570 & \textbf{0.2519} \\
                      & 2 $\rightarrow$ 3 & 0.2722          & 0.3491 & 0.2988 & 0.3061 & 0.3105 & 0.2817 & 0.2842 & 0.3569 & 0.3512  & 0.3421 & \textbf{0.2531} \\
                      & 2 $\rightarrow$ 4 & 0.3286          & 0.4238 & 0.3879 & 0.3752 & 0.3662 & 0.3447 & 0.3411 & 0.4391 & 0.4321  & 0.4046 & \textbf{0.2961} \\ \cmidrule{2-13} 
                      & 3 $\rightarrow$ 1 & 0.2744          & 0.3797 & 0.2967 & 0.2976 & 0.3305 & 0.2886 & 0.3238 & 0.4012 & 0.3880  & 0.3760 & \textbf{0.2740} \\
                      & 3 $\rightarrow$ 2 & 0.2532          & 0.3584 & 0.2707 & 0.2860 & 0.3059 & 0.2564 & 0.2960 & 0.3635 & 0.3577  & 0.3508 & \textbf{0.2472} \\
                      & 3 $\rightarrow$ 4 & 0.3240          & 0.4367 & 0.3572 & 0.3975 & 0.4084 & 0.3802 & 0.3650 & 0.4405 & 0.4227  & 0.4090 & \textbf{0.3202} \\ \cmidrule{2-13} 
                      & 4 $\rightarrow$ 1 & 0.2564          & 0.3819 & 0.2672 & 0.2833 & 0.3021 & 0.2592 & 0.3114 & 0.3549 & 0.3734  & 0.3517 & \textbf{0.2391} \\
                      & 4 $\rightarrow$ 2 & 0.2794          & 0.3910 & 0.3028 & 0.3212 & 0.3236 & 0.2952 & 0.3004 & 0.3388 & 0.3742  & 0.3599 & \textbf{0.2527} \\
                      & 4 $\rightarrow$ 3 & 0.2642          & 0.3537 & 0.2802 & 0.3108 & 0.3185 & 0.2831 & 0.2950 & 0.3797 & 0.3744  & 0.3396 & \textbf{0.2458} \\ \bottomrule
\end{tabular}%
}
\end{table*}